\documentclass{article} %
\usepackage{iclr2027_conference,times}

\usepackage{mathtools}
\usepackage{amsthm,amsmath,amsfonts,bm}

\newtheorem{proposition}{Proposition}
\newtheorem{corollary}{Corollary}
\newtheorem{theorem}{Theorem}

\def\eqref#1{equation~\ref{#1}}

\def\1{\bm{1}}

\def\ve{{\bm{e}}}

\def\vh{{\bm{h}}}

\def\vp{{\bm{p}}}

\def\vr{{\bm{r}}}

\def\vu{{\bm{u}}}

\def\vx{{\bm{x}}}

\DeclareMathAlphabet{\mathsfit}{\encodingdefault}{\sfdefault}{m}{sl}
\SetMathAlphabet{\mathsfit}{bold}{\encodingdefault}{\sfdefault}{bx}{n}

\newcommand{\KL}{D_{\mathrm{KL}}}

\DeclareMathOperator*{\argmax}{arg\,max}

\DeclareMathOperator{\supp}{supp}

\usepackage{algorithm}
\usepackage{algpseudocode}
\usepackage{hyperref}
\usepackage{url}
\usepackage{graphicx}
\usepackage{xstring}
\usepackage{xcolor}
\usepackage{makecell}
\usepackage{multirow}
\usepackage{float}
\usepackage{wrapfig}
\usepackage{booktabs,longtable,url,xcolor,amsmath}

\usepackage[most]{tcolorbox}

\newtcolorbox{promptbox}{
  colback=gray!10,
  colframe=gray!50,
  boxrule=0.5pt,
  arc=2pt,
  left=5pt,
  right=5pt,
  top=5pt,
  bottom=5pt
}

\newtcolorbox{generationbox}[1]{
  colback=white,
  colframe=gray!40,
  boxrule=0.5pt,
  arc=2pt,
  title=\textbf{#1},
  left=5pt,
  right=5pt,
  top=5pt,
  bottom=5pt
}

\newcommand{\authornote}[3]{{\color{#1}[\MakeUppercase{#2}: #3]}}
\newcommand{\note}[2][]{%
  \IfStrEq{#1}{}%
    {\PackageWarning{note}{Note: \detokenize{#2}}{\color{red}[#2]}}%
    {\PackageWarning{note}{Note by \detokenize{#1}: \detokenize{#2}}\authornote{red}{#1}{#2}}}

\makeatletter
\newcommand{\figwarntext}{}

\NewCommandCopy{\origincludegraphics}{\includegraphics}
\RenewDocumentCommand{\includegraphics}{O{} m}{%
  \IfEndWith{#2}{.pdf}{}{%
    \g@addto@macro\figwarntext{%
      \textcolor{red}{\textbf{WARNING: `\detokenize{#2}' is not a PDF.}}\ %
    }%
  }%
  \origincludegraphics[#1]{#2}%
}

\NewCommandCopy{\origcaption}{\caption}
\RenewDocumentCommand{\caption}{o m}{%
  \IfNoValueTF{#1}%
    {\origcaption{\figwarntext #2}}%
    {\origcaption[#1]{\figwarntext #2}}%
  \gdef\figwarntext{}%
}

\AddToHook{env/figure/begin}{\gdef\figwarntext{}}
\AddToHook{env/figure*/begin}{\gdef\figwarntext{}}
\AddToHook{env/table/begin}{\gdef\figwarntext{}}
\AddToHook{env/table*/begin}{\gdef\figwarntext{}}
\usepackage{amsthm}

\usepackage[capitalise]{cleveref}
\crefname{assumption}{Assumption}{Assumptions}

\renewcommand{\cite}{\citep}

\newcommand{\Dd}{\mathcal{D}}
\newcommand{\Ll}{\mathcal{L}}
\newcommand{\Uu}{\mathcal{U}}

\newcommand{\RR}{\mathbb{R}}
\newcommand{\EE}{\mathbb{E}}
\newcommand{\NN}{\mathbb{N}}
\renewcommand{\SS}{\mathbb{S}}

\newcommand{\penc}{\theta}
\newcommand{\Kpred}{K_\mathrm{pred}} %
\newcommand{\Naux}{N_\mathrm{aux}} %

\DeclareMathOperator{\Unif}{Unif}

\usepackage{booktabs}

\definecolor{myred}{RGB}{215,48,39}
\definecolor{mygreen}{RGB}{26,152,80}

\definecolor{chipgreen}{HTML}{d4ecd4}
\definecolor{chipgray}{HTML}{e8e8e8}

\definecolor{C0}{RGB}{31,119,180}
\definecolor{C1}{RGB}{255,127,14}
\definecolor{C2}{RGB}{44,160,44}
\definecolor{C3}{RGB}{214,39,40}
\definecolor{C4}{RGB}{148,103,189}
\definecolor{C5}{RGB}{140,86,75}
\definecolor{C6}{RGB}{227,119,194}
\definecolor{C7}{RGB}{127,127,127}
\definecolor{C8}{RGB}{188,189,34}
\definecolor{C9}{RGB}{23,190,207}

\title{Parallel Language Models\\via Variational Decoding}

\author{\textbf{Felix Draxler$^{1,*,\P}$, Robin C.~Geyer$^{2,3,*,\P}$, Daniel Geyfman$^{1,*}$, Farrin Marouf Sofian$^{1,*}$,} \\
\textbf{Justus Will$^{1,*}$, Philipp Hildebrandt$^{4}$, Giosue Migliorini$^{1}$, Prakhar Srivastava$^{1}$,} \\
\textbf{Gagan Sai Trinadh Duggaraju Peddiraju$^{1}$, Kushagra Pandey$^{1}$, Jorge da Silva Gon\c{c}alves$^{2}$,} \\
\textbf{Elijah Tavares$^{1}$, Jan Niklas Groeneveld$^{1}$, Yang Meng$^{1}$, Julia E.~Vogt$^{2}$, Stephan Mandt$^{1,\P}$} \\
\\
$^{1}$University of California, Irvine, CA, USA \\
$^{2}$Department of Computer Science, ETH Zurich, Switzerland \\
$^{3}$Institute for Artificial Intelligence in Medicine, Charité Berlin, Germany \\
$^{4}$Hasso Plattner Institute, Potsdam, Germany \\
\\
$^{*}$Equal contribution \quad $^{\P}$Project lead \quad See \cref{sec:contributions} for detailed contributions.\\
Correspondence: \texttt{mail@fdraxler.de}, \texttt{robin.geyer@charite.de}, \texttt{mandt@uci.edu}.
}

\iclrpreprint %
\begin{document}

\maketitle

\begin{abstract}
Autoregressive language models emit one token per forward pass. 
This causes high inference latency, in particular as models are scaled up for quality.
In this paper, we introduce Parallel Language Models (PLMs), which generate mutually consistent tokens in one forward pass by routing dependencies through latent auxiliary variables.
PLMs are trained using a variational bound that reduces to an autoregressive loss plus a parallel decoding loss.
Empirically, PLMs are on a faster and better Pareto frontier than autoregressive and discrete diffusion baselines.
Pretrained PLMs achieve near-autoregressive pretraining loss on model scales of 80M to 3B parameters, and speedup follows a power law in model size.
When post-trained, our 1.3B model matches or exceeds an identically trained autoregressive model on math, coding, and instruction-following benchmarks, speeding up inference $2.3\text{-}5.7\times$.
Together, we show that parallel generation can be trained from a single training objective at autoregressive quality.
We will make our code available at \url{https://github.com/mandt-lab/plm}.
\end{abstract}

\section{Introduction}

Large autoregressive models \cite{vaswani2017attention} are the core technical foundation of large language models (LLMs) \cite{brown2020language}. Autoregressivity comes with a significant downside, however: One model forward pass produces one token, causing high inference latency.

A plethora of approaches address this challenge, most prominently by increasing the speed of a given autoregressive model through smaller draft models in a process known as speculative decoding \cite{leviathan2023fast}. However, draft models do not reach the quality of their big counterparts, ultimately limiting the achievable speedup.
Discrete diffusion models \cite{hoogeboom2021argmax,austin2021structured} are an alternative modeling framework that predicts several tokens in one model call. 
However, a single model call cannot coordinate predicted tokens, requiring iterative refinement.
Few-step discrete diffusion models reduce the number of steps by amortizing iterative refinement \cite{monsefi2026fsdfm}.
Because diffusion models are more expensive to train, autoregressive models still dominate industry-scale inference.

This calls for a language modeling framework rooted in autoregression that can model arbitrary dependencies between predicted tokens in a single model call, combining effective training with fast inference at equal quality.
Recently, Parallel Token Prediction (PTP) presented a decoding framework that distills an existing autoregressive model into a fast decoder, predicting several tokens in parallel \cite{draxler2026parallel}. However, PTP was only trained as a speculative decoding method for an existing model, and the framework they proposed to train from sequence data alone is highly inefficient because it is iterative at training time.

In this work, we close this gap and propose Parallel Language Models (PLMs), a novel language modeling framework based on an autoregressive architecture that trains parallel LLMs directly from training data. In detail, we contribute:

\begin{itemize}
    \item We build Parallel Language Models based on a variational formulation, deriving effective pretraining and post-training strategies (\cref{sec:method}).
    \item We pretrain PLMs up to 3B parameters that are on a  faster and better Pareto frontier than autoregressive and discrete diffusion baselines (\cref{fig:performance-vs-speed}). Our scaling laws confirm near-autoregressive negative log-likelihoods, and show that the parallelism increases with model scale (\cref{sec:scaling-laws}).
    \item When post-trained, our 1.3B and 3B Parallel Language Models match or exceed autoregressive quality while answering $2.3\text{-}5.7\times$ faster (\cref{sec:post-training-comparison}).
\end{itemize}

\begin{figure}[t]
    \centering
    \includegraphics[width=\linewidth]{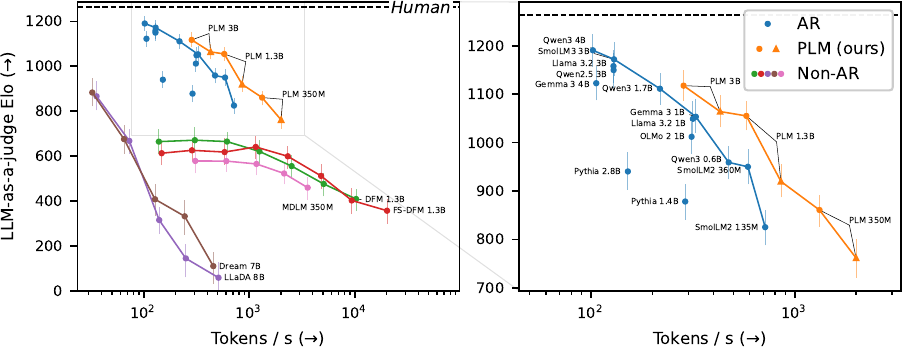}
    \caption{\textbf{Parallel Language Models (PLMs, \textcolor{C1}{orange}) shift the Pareto frontier of speed vs quality for language models trained from scratch.} \textit{(Right)}~At the same inference speed, PLMs outperform the quality of publicly available autoregressive models~(AR, \textcolor{C0}{blue}), despite the latter being trained on trillions of tokens and our models only trained on 25B~(PLM 350M) to 300B tokens~(PLM 3B). \textit{(Left)}~PLMs outperform discrete diffusion approaches at various NFEs. Inference speeds using \texttt{vLLM} on NVIDIA RTX 6000 Ada. PLM circle and triangle refer to \texttt{exact} respectively \texttt{relaxed} inference, see \cref{sec:inference}. Error bars are 95\% confidence intervals.}
    \label{fig:performance-vs-speed}
\end{figure}

\section{Related Work}

\paragraph{Parallel prediction and generation.}
Early non-autoregressive sequence models generate multiple output tokens in parallel, substantially reducing the number of sequential decoding steps \cite{gu2018nonautoregressive}.
More recently, multi-token prediction (MTP) predicts several future positions \cite{gloeckle2024better} and is used in large-scale autoregressive models such as DeepSeek-V3 \cite{deepseekai2025deepseekv3technicalreport}. These methods retain an autoregressive modeling framework rather than defining a joint parallel generative model.
Parallel Token Prediction (PTP) \cite{draxler2026parallel} instead introduces auxiliary random variables that enable multiple dependent future tokens to be generated jointly in one model call, and studies distillation from existing autoregressive models. \citet{samragh2025your} improve coherence of independent multi-token prediction by adding a small autoregressive head, and introduce Quadratic Decoding, which we adapt for PLM inference. Our work extends this line of research to pretraining of parallel language models directly from sequence data, developing a training recipe that scales to billions of parameters.

\paragraph{Iterative and few-step parallel generation.}
A complementary class of methods predicts many token positions in parallel and resolves dependencies through iterative refinement. Conditional Masked Language Models \cite{ghazvininejad2019mask} progressively fill and revise masked positions, while discrete diffusion language models such as MDLM \cite{sahoo2024simple}, LLaDA \cite{nie2025large}, and Dream \cite{ye2025dream} generate sequences through repeated denoising. Discrete Flow Matching \cite{gat2024discrete} provides a related framework for defining generative trajectories over discrete sequences, while Eso-LM \cite{sahoo2026esoteric} connects autoregressive and masked-diffusion modeling and improves diffusion inference. Recent work increasingly targets few-step generation: Infinite Mask Diffusion \cite{yoo2026infinite} and FS-DFM \cite{monsefi2026fsdfm} modify the generative process or training objective to improve performance under a small number of sampling steps, while flow-map approaches further compress iterative trajectories toward few- or single-step generation \cite{lee2026flow,potaptchik2026discrete}. Relatedly, Lookahead Decoding \cite{fu2024break} and Jacobi Forcing \cite{hu2026fast} generate tentative future tokens in parallel and recover consistency through verification or iterative refinement. In contrast, PLMs aim to model dependencies among simultaneously generated tokens directly within a single model call.

\paragraph{Autoregressive inference acceleration.}
A separate line of work accelerates inference from pretrained autoregressive language models. Speculative decoding uses a cheaper draft model to propose multiple tokens, which a target model verifies in parallel \citep{leviathan2023fast,chen2023accelerating}. Subsequent methods improve proposal quality through learned prediction heads or feature-based drafting, including Medusa \cite{cai2024medusa}, Hydra \cite{ankner2024hydra}, and EAGLE \cite{li2024eagle,li2024eagle2,li2025eagle3}. Unlike PLMs, these methods accelerate inference from a pre-existing autoregressive target model rather than training a parallel generative model directly from sequence data.

\paragraph{Latent-variable sequence models.}
Latent-variable models can also capture dependencies while retaining parallel sequence generation. FlowSeq \cite{ma2019flowseq} models non-autoregressive sequence generation through a flow-based latent distribution, while LaNMT \cite{shu2020latent} uses continuous latent variables and variational training for non-autoregressive translation. Parallel WaveNet \citep{oord2018parallel} similarly generates audio in a single pass by mapping per-timestep noise through an inverse autoregressive flow, but requires invertibility.
These approaches demonstrate how latent representations can mitigate the dependency limitations of factorized parallel decoders, but primarily target conditional sequence-to-sequence or audio generation. PLMs instead use token-level auxiliary variables to define a general language model in which dependencies among parallel predictions are mediated through the latent representation.

\section{Parallel Language Models}
\label{sec:method}

Autoregressive language models represent the distribution of language by explicitly predicting a histogram over possible choices from a token vocabulary $\{1, \dots, V\}$. This is an expressive paradigm, but it does not extend to predicting several tokens jointly because the possible token combinations $V^L$ scale exponentially in sequence length $L$.

One effective way to circumvent this is to use a language model that predicts sequences directly.
First, sample a random auxiliary $u_i \sim p(u_i)$ per token position $i = 1, \dots, L$. Then have a decoder map the vector of auxiliaries $\vu$ to a coherent sequence in one model call:
\begin{equation}
    \vx = f_\theta(\vu) \in \{ 1, \dots, V \}^L.
\end{equation}
Auxiliaries $u \in \Uu$ can either be continuous ($\Uu\subset\RR$) or discrete ($\Uu \subset \NN$), and we'll give more details in the subsequent sections.

Intuitively, the model generates different sequences for different random inputs. 
Prominent examples of such parallel decoders are Few-Step Discrete Flow Matching (FS-DFM) \cite{monsefi2026fsdfm} and Parallel Token Prediction (PTP) \cite{draxler2026parallel}. Both frameworks distill a parallel decoder from a pre-existing model, the former from a discrete diffusion model, and the latter from an autoregressive language model.

In the remainder of this section, we discuss the core framework of Parallel Language Models (PLMs). %
PLMs directly train on sequence data without relying on a pre-existing model. We provide a complete scalable recipe, covering pretraining, post-training, and inference.

\subsection{Self-Distilling Parallel Language Models on Sequence Data}
\label{sec:training-from-sequences}

\begin{figure}[t]
    \centering
    \includegraphics[width=\linewidth]{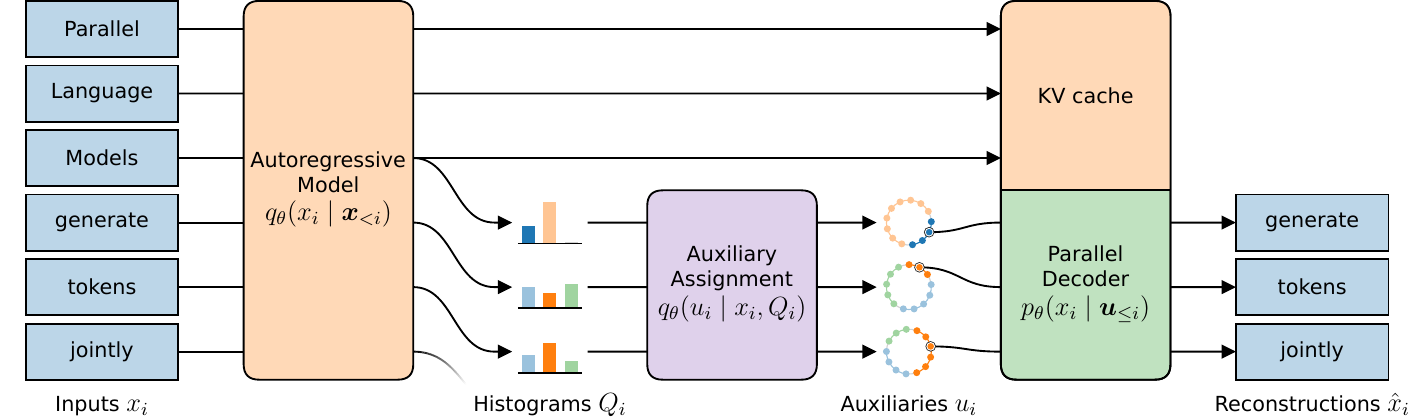}
    \caption{Parallel Language Models jointly train an encoder that maps training sequences $\vx$ to auxiliaries $\vu$, and a decoder that reconstructs them back to sequences $\hat \vx$ in parallel. The encoder leverages a standard autoregressive~(AR) model predicting next-token probabilities $Q_i = q_\theta(\,\cdot\mid\vx_{<i})$, which we then turn into auxiliaries $u_i \sim q_\theta(u_i \mid x_i, Q_i)$. See \cref{fig:attention-matrix} for KV cache details.}
    \label{fig:training}
\end{figure}

The Parallel Language Model we propose is a self-distillation method that jointly trains an autoregressive encoder $q_\penc(u_i \mid \vx_{\leq i})$ from sequences to auxiliaries, and a parallel decoder from auxiliaries to sequences $x_i = f_\theta(\vu_{\leq i})$, see \cref{fig:training}. Note that for notational simplicity, we remove dependency on a prompt or partially generated text in this derivation.

To train it, we derive our training objective from maximizing the likelihood of the training data under our model:
\begin{equation}
    \min_\theta -\EE_{\vx \sim p_\mathrm{data}}[\log p_\theta(\vx)].
    \label{eq:nll-loss}
\end{equation}

What is the probability of generating a sequence $\vx = \vx_{1:L}$ using a parallel decoder $x_i = f_\theta(\vu_{\leq i})$? It is given by integrating over all possible auxiliaries from their prescribed distribution $u_i \sim p(u_i)$, and then decoding them to the corresponding sequence:
\begin{equation}
    \log p_{\theta}(\vx) = \log \int \prod_{i=1}^L p(u_i) \, p_{\theta}(x_i \mid \vu_{\leq i}) \, d\vu.
    \label{eq:plm-generative-model}
\end{equation}
Note how we have switched from a deterministic decoder $f_\theta$ to a stochastic decoder $p_\theta$. This will allow taking gradients easily through a cross-entropy loss, and we take an $\argmax$ for inference (see \cref{sec:inference}). Crucially, the product distribution over future tokens ensures that we can predict tokens in parallel: All dependencies between tokens are described by the auxiliaries.

Optimizing the loss in \cref{eq:nll-loss} using the integral in \cref{eq:plm-generative-model} is intractable: if we were to approximate the integral using a finite number of random auxiliary samples, we would essentially never sample a given sequence even once. This is where the encoder $q_\penc(\vu \mid \vx) = \prod_{i=1}^L q_\penc(u_i \mid \vx_{\leq i})$ enters, proposing auxiliaries that should generate a given sequence. We rewrite \cref{eq:nll-loss} using the evidence lower bound (ELBO) \cite{kingma2013auto,zhang2018advances}, see \cref{app:elbo} for details:
\begin{align}
    -\log p_{\theta}(\vx) &
    \leq \EE_{\vu \sim q_\penc(\vu \mid \vx)}\!\left[ \sum_{i=1}^L -\log p_{\theta}(x_i \mid \vu_{\leq i}) \right] + \sum_{i=1}^L\KL(q_\penc(u_i \mid \vx_{\leq i}) \parallel p(u_i)).
    \label{eq:plm-elbo}
\end{align}
Remember that $p(u_i)$ specifies the distribution we sample the auxiliaries from during inference.

\Cref{eq:plm-elbo} is significantly easier to optimize than using \cref{eq:plm-generative-model}, since the encoder proposes auxiliaries $\vu$ for decoding a given sequence $\vx$. We show in \cref{app:no-posterior-collapse} that equality holds at the global minimum, and that we can get arbitrarily close to this global minimum with a rich enough model, and so optimizing \cref{eq:plm-elbo} actually optimizes \cref{eq:nll-loss}.

\paragraph{Auxiliary design} We now specify how we construct the auxiliary space. We formulate two desiderata that do not change the optimum of \cref{eq:plm-elbo} yet make the design of auxiliaries simple and flexible: (a)~Different tokens in the same context should be encoded to different auxiliaries. This allows decoding to be deterministic given auxiliaries and thus parallelizable. (b)~The share of auxiliaries assigned to a given token should be equal to its probability given the preceding tokens.

These properties can easily be fulfilled by splitting the encoder into a standard autoregressive model predicting the next-token distributions $Q_i = q_\theta(\,\cdot\mid\vx_{<i})$, and then assigning auxiliaries $q_\theta(u_i \mid x_i, Q_i)$ using that probability vector $Q_i$, with entries $Q_i(v)=q_\theta(x_i=v\mid\vx_{<i})$, and the token $x_i$. Concretely, we achieve the latter by designing an oracle decoding function $x_i = f(u; Q_i)$ that specifies which auxiliaries $u \in \Uu$ map to which tokens. This oracle uniquely defines a set of auxiliaries encoding to that token $\Uu_{x_i} = \{ u \in \Uu : x_i = f(u; Q_i) \}$. Then, the auxiliary encoder samples a random $u_i \in \Uu_{x_i}$ by restricting the auxiliary prior $p(u_i)$ to these allowed auxiliaries:
\begin{equation}
    q_\theta(u_i \mid x_i, Q_i) \propto p(u_i)\, \1[u_i \in \Uu_{x_i}].
    \label{eq:shape-matching}
\end{equation}
In this formulation, the disjointness and relative mass properties above read:
\begin{equation}
    (a) \quad \Uu_{x_i} \cap \Uu_{x_i'} = \emptyset \quad \text{for all } x_i \neq x_i'.
    \qquad
    (b) \quad p(u_i \in \Uu_{x_i}) = q_\theta(x_i \mid \vx_{<i}),
    \label{eq:plm-auxiliary-association-properties}
\end{equation}

For any oracle $f(u; Q_i)$ that fulfills these properties, \cref{eq:plm-elbo} can still be minimized given enough model capacity (cf. \cref{app:optimal-decoder}). For example, PTP \cite{draxler2026parallel} leveraged the quantiles of the histogram of a (fixed) autoregressive model $f^\text{PTP}(u; Q_i) = \text{ICDF}_{Q_i}(u)$, fulfilling the conditions in \cref{eq:plm-auxiliary-association-properties}.
In \cref{sec:auxiliaries}, we give our improved design choice for associating tokens and auxiliaries that follows \cref{eq:plm-auxiliary-association-properties}.

Crucially, if an encoder is constructed as above, and %
the shape of $q_\penc(u_i \mid \vx_{\leq i})$ matches the shape of $p(u_i)$ as in \cref{eq:shape-matching}, we find that:
\begin{equation}
    \label{eq:nll-bound}
    \KL(q_\penc(u_i \mid \vx_{\leq i}) \parallel p(u_i)) = -\log q_\penc(x_i \mid \vx_{<i}).
\end{equation}
In other words, the trainable part of the encoding is a standard autoregressive model, and we design the association between tokens and auxiliaries. Proof in \Cref{app:encoder-loss}.

Using the autoregressive encoder design, we find that computing gradients through the sampling of $\vu \sim q_\theta(\vu \mid \vx)$ is not needed as the optimum of \cref{eq:plm-elbo} does not change. We thus do not apply the reparameterization trick common in VAEs. This trivially allows using discrete auxiliary spaces, and stabilizes training empirically.

\paragraph{PLM training objective} Together, we obtain the following training objective for a Parallel Language Model, where $\mathrm{sg}[\cdot]$ implies a stopped gradient:
\begin{equation}\label{eq:loss}
    \Ll(\theta; \vx) = \underbrace{\EE_{\mathrm{sg}[\vu \sim q_\penc(\vu \mid \vx)]}\!\left[ \sum_{i=1}^L -\log p_{\theta}(x_i \mid \vu_{\leq i}) \right]}_{\Ll_\mathrm{dec}(\theta; \vx)} + \underbrace{\sum_{i=1}^L -\log q_\penc(x_i \mid \vx_{<i})}_{\Ll_\mathrm{enc}(\penc; \vx)}.
\end{equation}
It composes an autoregressive encoder with a parallel decoder. The above derivation confirms that this is a coherent way to train a parallel model. Next, we discuss its efficient implementation.

\subsection{Scalable Training}
\label{sec:scalable-implementation}

We now discuss practical design choices for training a Parallel Language Model that scales to billions of parameters and tokens.
At a high level, training \cref{eq:loss} on a given sequence $\vx_{1:L}$ means first encoding every token $x_i$ to its auxiliary $u_i$ with the encoder $q_\penc$, and then the decoder $p_\theta$ reconstructs it to $\hat x_i$, see \cref{fig:training}. 
Below is our recipe to make this efficient:

\paragraph{Weight-sharing encoder and decoder} Since a large part of the difficulty in predicting language is in parsing context, we share the model weights between the encoder $q_\penc$ and the decoder $p_\theta$. %
They only differ in the embedding layer, which embeds tokens from the vocabulary differently from the auxiliaries.

\paragraph{Block-autoregressive decoding} 
Empirically, we find that reconstruction becomes increasingly difficult with sequence length. Therefore, we cap the maximum decoding sequence length to a block size $K$. Concretely, for a training sequence of length $L>K$, we divide the input sequence $\vx$ into blocks of $K$ tokens. So to predict a block of tokens $\hat \vx_{i_z:i_z + K-1}$ starting at position $i_z$, the decoder receives the corresponding auxiliaries $\vu_{i_z:i_z + K-1}$ and the input tokens $\vx_{<i_z}$ before the block. To avoid any recomputation on these tokens before the decoding block, we let the decoder attend to the KV cache of the encoder, and pass gradients to it (but not through the sampling of auxiliaries).
\Cref{fig:attention-matrix} shows how this can be efficiently implemented with an attention mask that makes each decoder block attend to the encoder tokens up to the block start. We train all our models with block size $K=16$ tokens.

\begin{wrapfigure}[17]{r}{0.4\textwidth}
    \centering
    \includegraphics[width=\linewidth]{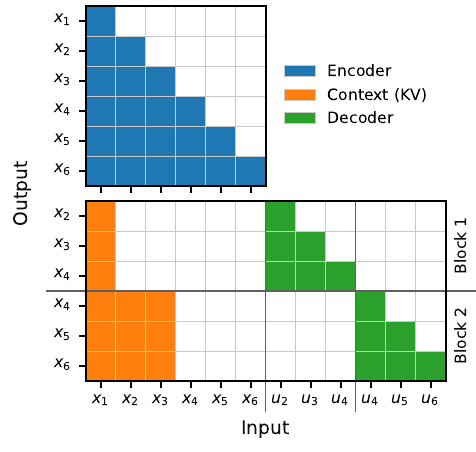}
    \caption{\textit{(Top)} The encoder $q_\penc$ has a standard causal attention mask (\textcolor{C0}{blue}). \textit{(Bottom)} The decoder $p_\theta$ attends to the input tokens before its block. For example, the first block attends to the first input token $x_1$ via the encoder (\textcolor{C1}{orange}), and predicts $\vx_{2:4}$ at the positions receiving auxiliaries $\vu_{2:4}$ (\textcolor{C2}{green}).}
    \label{fig:attention-matrix}
\end{wrapfigure}
\paragraph{Encoder-to-decoder ratio} While the optimum of \cref{eq:loss} can encode and decode perfectly at infinite model capacity, in practice model capacity is split between the two modes of the model. We therefore introduce a loss ratio $\beta$, analogous to the loss weight in a $\beta$-VAE \cite{higgins2017betavae}. We implement this ratio by not decoding all tokens that were encoded ($\beta > 1$), or by decoding every token several times ($\beta < 1$). For $\beta = 1$, every encoded token is decoded once. To this end, we use a stride $s = \lfloor \beta K \rfloor$ between blocks, so $\beta \approx s/K$. \Cref{fig:attention-matrix} shows an example for stride $s = 2$ at block size $K = 3$.
The explicit loss becomes, with $u^{(z)}_i \sim q_\penc(u_i \mid \vx_{\leq i})$ sampled once per block without gradient, block count $Z = \lceil L / s \rceil$ and block offsets $i_z=(z-1)s + 1$:
\begin{align}
    \Ll(\theta; \vx) =& \underbrace{\sum_{z=1}^{Z} \sum_{k=0}^{K-1} -\log p_{\theta}(x_{i_z + k} \mid \vx_{< i_z},\, \vu^{(z)}_{i_z : i_z + k})}_{\Ll_\mathrm{dec}(\theta; \vx)} \nonumber\\&+ \underbrace{\sum_{i=1}^L -\log q_\penc(x_i \mid \vx_{<i})}_{\Ll_\mathrm{enc}(\penc; \vx)}.
    \label{eq:plm-final-objective}
\end{align}
Random shifts of the block grid and truncation at sequence boundaries are described in \cref{app:pretraining-details}.
We denote models by PLM-$\beta$; the autoregressive baseline is recovered in the limit $\beta\to\infty$.

\paragraph{Geometry-aware auxiliaries}
\label{sec:auxiliaries}
\begin{wrapfigure}[11]{r}{0.25\textwidth}
 \centering
 \small
 \setlength{\abovecaptionskip}{3pt}
 \includegraphics[width=\linewidth]{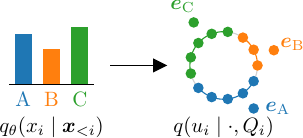}
 \caption{Auxiliaries are grouped around token embeddings $\ve_v$, with group sizes matching token probabilities.
 }
  \label{fig:encoding}
\end{wrapfigure}
We now describe an efficient and effective way to design the auxiliary space. We decide to leverage a set of $\Naux$ unit vectors living in the embedding space of the transformer $\ve_u \in \SS^{d-1}, \,u=1, \dots, \Naux$ as auxiliaries. The prior is uniform over these indices, $p(u)=1/\Naux$. Unlike quantiles over an arbitrary vocabulary order as in PTP \cite{draxler2026parallel}, embedding-space auxiliaries naturally exploit token similarity (see ablation in \cref{app:ablation}).

For assigning auxiliaries to tokens, we choose to assign every auxiliary $u$ to the most similar regular token embedding $\ve_v$, as measured by dot-product similarity in the embedding space. To ensure that token probabilities $Q_i=q_\theta(x_i \mid \vx_{<i})$ match the relative mass of auxiliaries (property (b) in \cref{eq:plm-auxiliary-association-properties}), we adjust this similarity using a scalar bias $w_v$ for each token
(see \cref{sec:aux-weight-fitting}):
\begin{equation}
    f^\text{Voronoi}(u; Q_i) = \argmax_v \{ \ve_v\cdot\ve_u + w_v(Q_i) \}.
    \label{eq:aux-scores}
\end{equation}
This is a discretized version of a weighted Voronoi tessellation \cite{aurenhammer1987power} on the sphere, see \cref{fig:encoding} for an example with three tokens and a 2-dimensional embedding space. We initialize the auxiliary embeddings randomly and leave them fixed during training.

At inference, we insert the embedding $\ve_u$ of a uniformly sampled $u \in \{1, \dots, \Naux\}$ into the parallel decoder.

Note that a finite number of $\Naux$ auxiliaries has a limited resolution of representing token probabilities. We choose $\Naux=512$ and feed a mask auxiliary for any token that is not among the top-128 predicted by the AR model. During training, the mean total variation between the fitted cell masses and the rounded target masses was below 0.3\% (see \cref{app:pretraining-details}).

\paragraph{Representation alignment}
Each training step looks at the sequence twice: once autoregressively, and once at parallel decoding. We find that adding a representation alignment loss \cite{yu2025repa} that aligns the hidden state of the parallel decoder to the autoregressive encoder significantly speeds up training, achieving the same decoding quality at approximately half the training compute (see \cref{app:ablation}). We use the following regularization:
\begin{equation}
    \Ll_\mathrm{align}(\theta, \psi; \vx) = \sum_{z=1}^{Z} \sum_{k=0}^{K-1} \Bigl( 1 - \cos\bigl( g_\psi\bigl(\vh^{(z)}_{i_z + k}\bigr), \mathrm{sg}\bigl[\vh_{i_z + k}\bigr] \bigr) \Bigr).
    \label{eq:align-loss}
\end{equation}
Here, $\vh^{(z)}_{i}$ is the hidden state of the parallel decoder at the position that predicts $x_i$ in block $z$, and $\vh_i$ is the autoregressive state that predicts the same token after reading $\vx_{<i}$. We read them at layers $\ell_\mathrm{dec} < \ell_\mathrm{enc}$ (values in \cref{app:pretraining-details}). This asks the decoder to arrive early at what the encoder reaches later. The stop gradient $\mathrm{sg}[\cdot]$ makes the decoder chase the encoder and never the reverse, and the learned affine map $g_\psi$ facilitates alignment. We add $\lambda_\mathrm{align}\Ll_\mathrm{align}$ to \cref{eq:plm-final-objective} and drop $g_\psi$ after training, so inference is unchanged.

\subsection{Post-Training}
\label{sec:post-training}

We post-train PLMs in two stages. First, supervised fine-tuning (SFT)
adapts the model to instruction--response data using the joint
encoder--decoder objective of \cref{eq:plm-final-objective}, with the loss
restricted to response tokens so that the prompt serves only as context. We train on a
mixture of general instruction-following, mathematics, coding, and precise
instruction-following data. Dataset and optimization details are given in
\cref{app:sft-details}.

We then apply Direct Preference Optimization (DPO)
\cite{rafailov2023direct} jointly to the two modes of the PLM: the
autoregressive encoder and the parallel decoder. Let $\theta$ denote their
shared parameters and $\theta_{\mathrm{ref}}$ a frozen SFT reference. For a
prompt--response sequence $\vx=(\vp,\vr)$, let $\Ll_{\mathrm{enc}}^{\vr}(\theta;\vx)$ and $\Ll_{\mathrm{dec}}^{\vr}(\theta;\vx,\vu)$ denote the response-restricted encoder and decoder components of \cref{eq:plm-final-objective}. Here, $\Ll_{\mathrm{dec}}^{\vr}(\theta;\vx,\vu)$ denotes the response-restricted decoder loss of \cref{eq:plm-final-objective} for a fixed auxiliary realization $\vu$, rather than averaging over newly sampled auxiliaries. We define the response score
\begin{equation}
s_\theta(\vx,\vu)
=
-
\frac{
1
}{
L^{\vr}(1 + 1/\beta)
}(\Ll_{\mathrm{enc}}^{\vr}(\theta;\vx)
+
\Ll_{\mathrm{dec}}^{\vr}(\theta;\vx,\vu)).
\label{eq:dpo-plm-score}
\end{equation}

The fraction normalizes by the total number of terms, preventing the response length $L^{\vr}$ from dominating the score.

DPO trains on preference triples $(\vp,\vr^+,\vr^-)$, where $\vr^+$ is the
preferred and $\vr^-$ the rejected response to the prompt $\vp$. Write \(\vx^\pm=(\vp,\vr^\pm)\). For each of
the two candidate responses, we sample auxiliaries from the frozen reference encoder,
$
\vu^\pm\sim q_{\theta_{\mathrm{ref}}}(\vu\mid\vx^\pm),
$
where each auxiliary is drawn from the
observed target token's Voronoi cell. The same candidate-specific auxiliaries are used to evaluate the policy and reference. They remain samples from the frozen reference encoder as the policy changes, keeping the comparison stationary conditional on the sampled auxiliary realization.

Define
$
\delta_\theta(\vx,\vu)
=
s_\theta(\vx,\vu)
-
s_{\theta_{\mathrm{ref}}}(\vx,\vu).
$
We then optimize
\begin{equation}
\Ll_{\mathrm{DPO}}(\theta)
=
-
\EE_{\substack{
(\vp,\vr^+,\vr^-)\sim\Dd_{\mathrm{pref}}\\
\vu^+\sim q_{\theta_{\mathrm{ref}}}(\cdot\mid\vx^+)\\
\vu^-\sim q_{\theta_{\mathrm{ref}}}(\cdot\mid\vx^-)
}}
\log\sigma\left(
\gamma
\left[
\delta_\theta(\vx^+,\vu^+)
-
\delta_\theta(\vx^-,\vu^-)
\right]
\right),
\label{eq:plm-dpo-loss}
\end{equation}
where $\gamma$ is the DPO inverse temperature. For the autoregressive
baseline, we use the same objective with the decoder component omitted from
\cref{eq:dpo-plm-score}.

Unlike canonical DPO, which uses summed sequence log-probabilities, our PLM score replaces the intractable marginal response log-likelihood with the auxiliary-conditioned composite training score in \cref{eq:dpo-plm-score} and normalizes it by the number of prediction terms. We therefore optimize a length-normalized DPO-style surrogate. This normalization is similar to the length-normalized DPO of \citet{lambert2025tulu}. SimPO \citep{meng2024simpo} uses the same normalization but drops the reference model, which we retain. Further details are given in \cref{app:dpo-details}. %

\subsection{Inference}
\label{sec:inference}

During inference, tokens are generated in parallel via Quadratic Decoding \cite{samragh2025your}, whose two main components, proposal and verification, we adapt for PLMs. First, a block of $\Kpred$ coordinated tokens $x_j = f_\theta(\vx_{<i}, \vu_{i:j})$, $i \leq j < i+\Kpred$ is produced by feeding independently randomly drawn auxiliaries $u_j \sim p(u_j)$ (cf.~\cref{sec:auxiliaries}) into our trained model $p_\theta$ and then taking the $\argmax$ of the produced logits. The inference block size $\Kpred$ is set at inference time and can be smaller than the training block size $K$. Second, the proposed tokens are verified by the encoder against a predefined rule and a subset of accepted tokens is appended to the context. While our model is trained to predict $K$ tokens, this step is necessary, as in practice finite model capacity degrades proposal quality with increased distance to the context. Quadratic Decoding allows both steps to be done in a single efficient model call, by making $\Kpred$ proposals of which only one is selected after verification. While this widened model call increases the computational cost, for small $\Kpred$ the inference overhead is negligible (cf. \Cref{tab:post-train-spec}), as the extra work is efficiently parallelized on one GPU by adopting a custom attention matrix. See \cref{app:inference} for algorithmic details.

Generally, a verification rule is set to guarantee high quality samples. In speculative decoding, based on an existing AR model, this is usually achieved by forcing the decoding algorithm to produce the same sample distribution as the reference. PTP \cite{draxler2026parallel}, for example, forces the generation to be token-identical to a random rollout of the AR model. Here, we argue that this strict verification artificially harms throughput. If our model creates a plausible continuation, we avoid discarding it by loosening our acceptance rule to also admit all tokens that have high likelihood under our encoder, specifically ones that are in the $\rho$-nucleus of $q_\theta(x_i \mid \vx_{<i})$ or ones with $q_\theta(x_i \mid \vx_{<i}) \geq \epsilon$. \Cref{tab:post-training} shows this increases the number of accepted tokens per call of the post-trained models by an average of $46.5\%$ at comparable quality.

\section{Experiments}
\label{sec:experiments}

We empirically evaluate Parallel Language Models (PLMs) by their quality and wall-clock speed compared to other language model frameworks.

To this end, we pretrain PLMs of 80M, 150M, 350M, 1.3B and 3B parameters on a general mix of tokens. First, we fit scaling laws to PLMs. We find that the text quality of the encoder follows that of an identically trained autoregressive model, and that the average number of tokens the decoder can reconstruct in a single model call increases with model size, see \cref{sec:scaling-laws}. Compared to competing language model frameworks, PLMs achieve significant speedups (\cref{sec:pretraining-comparison}). Finally, we post-train and evaluate PLMs on math, coding, and instruction-following benchmarks, finding comparable or better quality compared to identically trained autoregressive models, but at about $2.3\text{-}5.7\times$ the speed.

\subsection{Pretraining}

\begin{figure}
    \centering
    \includegraphics[width=\linewidth]{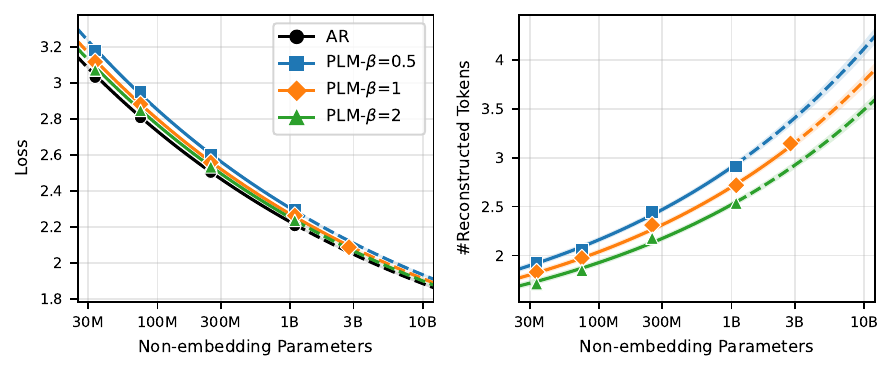}
    \caption{\textbf{PLMs become better and more parallel with scale.} \textit{(Left)} PLMs (colored) achieve comparable quality to an AR model of the same size (black). \textit{(Right)} PLMs achieve more correct tokens with greater model size. Varying the encoder-to-decoder ratio $\beta$ reduces quality and increases speed. Markers show measurements at token to non-embedding parameter ratio $D/N=90$, solid lines are \cref{eq:scaling-laws}.
    }
    \label{fig:scaling-laws}
\end{figure}

\subsubsection{Scaling Laws}
\label{sec:scaling-laws}

We study how language quality and parallelism improve with PLM size and training tokens. Quality is measured by the encoder's negative log-likelihood per token, $\Ll_\mathrm{enc}$, as is standard for autoregressive models. Our parallelism metric $R$ estimates the average number of training tokens correctly reconstructed in parallel per model call (see \cref{app:formal-reconstructed-tokens} for a formal definition). An autoregressive model has $R=1$.

We pretrain PLM-$\beta$=0.5, PLM-$\beta$=1 and PLM-$\beta$=2, corresponding to encoder-to-decoder ratios $\beta=0.5,1,2$ (\cref{sec:scalable-implementation}), at sizes of 80M, 150M, 350M and 1.3B parameters on ClimbMix \cite{diao2025climb}. We also train matching autoregressive baselines and add a 3B PLM-$\beta$=1 extrapolation. All PLMs use decoder blocks of $K=16$ tokens. Token budgets reach $D/N \approx 100$, where $N$ denotes non-embedding parameters and $D$ training tokens. Hyperparameters follow Step Law \cite{li2025predictable}, with details in \cref{app:pretraining-details}.

\Cref{fig:scaling-laws} shows that \textbf{larger models reconstruct more tokens at a lower loss}. We fit Chinchilla-style laws \cite{hoffmann2022training} jointly to AR and PLM runs, modeling the loss through a normalized effective model size $n_\mathrm{eff}= n\beta/(\beta+0.48)$ and $R$ through a power law above the autoregressive baseline:
\begin{equation}
  \Ll_\mathrm{enc} \approx \Ll_\mathrm{AR}(n_\mathrm{eff}, d),
  \quad
  R \approx 1 + 1.05\,
  \frac{n^{0.120}\,d^{0.094}}{\beta^{0.16}},
  \label{eq:scaling-laws}
\end{equation}
\begin{equation}
  \Ll_\mathrm{AR}(n, d) \approx 1.105
  + \frac{0.851}{n^{0.194}}
  + \frac{0.763}{d^{0.127}}.
\end{equation}
Here, $n=N/10^8$ and $d=D/10^{10}$ express non-embedding parameters in units of 100M and training tokens in units of 10B. Both laws recover the autoregressive case as $\beta\to\infty$.

Over the fitted range, at fixed $\beta$, the absolute loss gap to an AR model of the same size therefore shrinks as models grow, while $R$ increases with both model size and training tokens. Thus, as depicted in \cref{fig:scaling-laws}, \textbf{PLMs become better and more parallel with scale}, and $\beta$ controls the quality--parallelism trade-off. Fit details, uncertainty estimates and the 3B extrapolation check are in \cref{app:scaling-law-fits}.

\subsubsection{Pretraining Model Comparison}
\label{sec:pretraining-comparison}

Parallel Language Models compete with other language models that allow predicting more than one token per model call. However, the latter often cannot model dependencies between tokens that are predicted in the same step, or require substantially longer training.

We therefore compare the speed and quality of publicly available autoregressive models, discrete diffusion models, and few-step discrete diffusion models to our PLMs at $D/N=100$ from \cref{sec:scaling-laws}. To this end, we ask a variety of models to continue the same 200 prompts for an additional 128 tokens, where each prompt is a 64-token opening of a randomly sampled FineWeb-Edu document \cite{lozhkov2024fineweb-edu}. Quality is measured using Phi-4 \cite{abdin2024phi} as a judge, and we fit Bradley-Terry Elo scores on the resulting data \cite{hunter2004mm}.

\Cref{fig:performance-vs-speed} shows that our framework \textbf{sets a new standard on the Pareto frontier of speed vs quality for pretrained models}, despite having trained on substantially fewer tokens than commercially trained models. PLMs also \textbf{outperform all tested discrete diffusion models, including few-step models}. The AR models use temperature $1$ sampling among the top $128$ tokens for a fair comparison to the geometry-aware auxiliaries used in our PLM model, as described in \Cref{sec:scalable-implementation}. We also include the actual FineWeb-Edu document completion as a human reference, which consistently scores higher than any model in the benchmark. 

We benchmark PLM with strict speculative decoding, and relaxed speculative decoding as in \Cref{sec:inference}. \Cref{fig:performance-vs-speed} shows that this relaxed acceptance rule gives an increase in speed for a small drop in quality within a fixed model size. Additionally, as we increase the model size, the PLMs' ability to predict more tokens per call increases, going from 1.95 during exact speculative decoding at 350M to 2.26 at 3B. 

In \cref{app:ablation,sec:pretraining-eval-details}, we report additional results and provide details on the experimental setup.

\subsection{Post-Training}
\label{sec:post-training-comparison}

We now confirm that Parallel Language Models can be fine-tuned to specific tasks, just like autoregressive models. To this end, we post-train our 1.3B and 3B PLMs using SFT and DPO, see \cref{sec:post-training}. We also post-train our purely autoregressive 1.3B baseline.

\cref{tab:post-training-short} shows that our post-trained PLMs match or exceed the quality of an identically post-trained 1.3B autoregressive baseline, all while sampling approximately $2.3\text{-}5.7\times$ faster, depending on the task. In addition, we confirm that our 3B PLM improves quality further, again sampling significantly faster than an autoregressive model of the same size would.

\begin{table}[t]
  \centering
  \caption{\textbf{Post-trained 1.3B PLMs are comparable to or better than an autoregressive model at $2.3\times$ to $5.7\times$ the inference speed.} The 3B PLM-$\beta=1$ model achieves better quality than all 1.3B variants. ``Q / s'' is questions per second. $^\dagger$We do not train an autoregressive model, so we use the 3B PLM in autoregressive inference mode for comparison.}
  \label{tab:post-training-short}
  \small
  \resizebox{\linewidth}{!}{\begin{tabular}{lllcccccccc}
\toprule
\multirow{2}{*}[-\dimexpr(\aboverulesep+\belowrulesep+\cmidrulewidth)*1/2\relax]{Size} & \multirow{2}{*}[-\dimexpr(\aboverulesep+\belowrulesep+\cmidrulewidth)*1/2\relax]{$\beta$} & \multirow{2}{*}[-\dimexpr(\aboverulesep+\belowrulesep+\cmidrulewidth)*1/2\relax]{Decoding} & \multicolumn{2}{c}{GSM8K} & \multicolumn{2}{c}{MATH-500} & \multicolumn{2}{c}{HumanEval} & \multicolumn{2}{c}{IFEval} \\
\cmidrule(lr){4-5}\cmidrule(lr){6-7}\cmidrule(lr){8-9}\cmidrule(lr){10-11}
 & & & Score & Q / s & Score & Q / s & Score & Q / s & Score & Q / s \\
\midrule
\multirow{4}{*}{1.3B} & AR & AR & 27.9 & 0.9 (1.0\texttimes) & 7.4 & 0.4 (1.0\texttimes) & 12.8 & 0.7 (1.0\texttimes) & \textbf{48.2} & 0.8 (1.0\texttimes) \\
 & $0.5$ & Relaxed & \textbf{32.6} & \textbf{5.2 (5.7\texttimes)} & \textbf{8.4} & \textbf{1.4 (3.7\texttimes)} & \textbf{17.7} & 3.0 (4.5\texttimes) & 45.1 & \textbf{1.9 (2.4\texttimes)} \\
 & $1$ & Relaxed & 29.4 & 4.9 (5.4\texttimes) & 7.2 & 1.2 (3.1\texttimes) & 10.4 & \textbf{5.2 (7.8\texttimes)} & 45.7 & \textbf{1.9 (2.5\texttimes)} \\
 & $2$ & Relaxed & 29.5 & 4.8 (5.2\texttimes) & \textbf{8.4} & 1.0 (2.8\texttimes) & 15.8 & 3.0 (4.5\texttimes) & 48.1 & 1.8 (2.3\texttimes) \\
\midrule
\multirow{2}{*}{3B} & \multirow{2}{*}{$1$} & AR$^\dagger$ & \textbf{39.0} & 0.4 (1.0\texttimes) & 12.2 & 0.2 (1.0\texttimes) & 23.2 & 0.4 (1.0\texttimes) & \textbf{54.0} & 0.4 (1.0\texttimes) \\
 &  & Relaxed & 37.4 & \textbf{2.1 (4.8\texttimes)} & \textbf{14.8} & \textbf{0.7 (3.8\texttimes)} & \textbf{32.3} & \textbf{1.5 (4.1\texttimes)} & 53.6 & \textbf{1.0 (2.5\texttimes)} \\
\bottomrule
\end{tabular}
}
\end{table}

\section{Conclusion}

This paper derives a complete and scalable recipe to build language models from scratch that predict coordinated tokens in parallel. Already at the model scales and token budgets trained for this paper, there is a clear inference speed advantage over autoregressive models of equal quality, even over those trained with significantly higher budget.

Going forward, our mathematical framework and our scaling laws indicate no fundamental boundaries that would prevent scaling our framework to industry scales. Instead, they predict that scaled up models achieve autoregressive quality, in particular after post-training, and generate many tokens in parallel. We hope this opens a research avenue that enables autoregressive models to be as parallel in inference as they already are in training.

\subsection*{AI Use Statement}

In this work, we used generative AI tools for 
formulating mathematical claims,
assisting in the writing of proofs design, 
and
implementing methods.
We have not used generative AI tools for
helping develop theoretical models or conceptual frameworks,
providing critical ingredients for proving mathematical claims, or
interpreting results.
The categories
generating synthetic data sets,
proposing or refining hypotheses,
assisting with translation,
cleaning and reformatting datasets, and
supporting qualitative and thematic data analysis
are not applicable to this work. 

Additionally, we used generative AI tools for 
creating or modifying scientific figures or images,
creating or editing software code,
drafting parts of our research paper,
summarizing or analyzing existing literature,
sourcing/searching for information, 
editing our research paper to improve readability, and
identifying relevant literature.
We have reviewed all AI-assisted work by verifying LLM-generated code by two authors, verifying information, and by checking for potential plagiarism for research ideas through manual literature survey.
We take responsibility for the final content of this work, including text, claims or artifacts produced with the aid of generative AI.

\subsection*{Ethics Statement}

The authors do not feel that their paper submission raises any questions \ificlrcameraready regarding the ICLR Code of Ethics \fi beyond the general risks inherent to developing and distributing AI models.

\subsection*{Reproducibility Statement}

We have included all details necessary to replicate our experiments in \cref{app:pretraining-details,app:post-training-details}, and all details necessary to derive our method are in \cref{app:derivations}. We will publish all our code upon publication.

\ificlrfinal

\subsubsection*{Acknowledgments}
Stephan Mandt acknowledges support from the National Science Foundation (NSF) through CAREER Award IIS-2047418 and Grant IIS-2611549. The authors acknowledge the Swiss AI Initiative for providing the computational resources required for this work. The computations were performed on the Clariden vCluster on the Alps supercomputing infrastructure operated by the Swiss National Supercomputing Centre (CSCS) under project IDs ab014 and a0200. 
RG and JS acknowledge support from StimuLOOP (\#1-007811-002) grant and the Vontobel Foundation. RG acknowledges funding from the German Federal Ministry of Research, Technology and Space under the grant BIFOLD26B. 
PH acknowledges support from the Transatlantic Collaboration Fund (TCF), funded by the Hasso Plattner Foundation (HPF).
\fi

\bibliography{iclr2027_conference}
\bibliographystyle{iclr2027_conference}

\clearpage

\appendix
\crefalias{section}{appendix}
\crefalias{subsection}{subappendix}
\crefalias{subsubsection}{subsubappendix}

\section{Author Contribution}
\label{sec:contributions}

\textbf{Felix Draxler} designed and coordinated the overall project, coauthored the compute grant, and managed computational resources. He headed the pretraining team, contributed to pretraining models and coauthored the scaling laws. He proposed the idea of self-distillation, and contributed to the construction of the ELBO loss derivation.

\textbf{Robin C. Geyer} led external partnerships and set technical direction, coauthored the compute grant, and managed computational resources. He was co-lead and core developer on the pretraining team and coauthored the scaling-law analysis. He moved beyond representing auxiliary variables as scalars by developing a Voronoi-based method that maps them to vectors in the transformer’s embedding space.

\textbf{Farrin Marouf Sofian} led the post-training team, co-designed
the post-training recipe, including algorithms, datasets and benchmarks. Implemented supervised fine-tuning. She also co-developed the VAE framework for PLM as presented in the paper and co-edited and co-revised the final manuscript.

\textbf{Daniel Geyfman} led the baseline/evaluation team and conducted both pretraining and posttraining quality experiments. He also contributed to getting the initial pretrained model and running initial ablations.

\textbf{Justus Will} led the inference team. He proposed, implemented, analyzed, and tested the inference algorithms PLM uses. He was responsible for designing the fast hardware-optimized inference codebase and porting the trained models into it. He advised on several design choices across the project and derived the exact conditions and proofs in \Cref{app:derivations}.   

\textbf{Philipp Hildebrandt} was a core developer on the pretraining team and contributed to developing self-distillation into a practical training recipe, explored and implemented representation alignment to improve training efficiency and built the experimental setup for the scaling-law study.

\textbf{Giosue Migliorini} was a core developer on the pretraining team, contributing to incorporate PLM-specific training, attention masking, and metrics into the Megatron repository. He also contributed to the variational interpretation of PLMs, and to testing novel inference recipes.

\textbf{Prakhar Srivastava} contributed to the design and implementation of post-training methods for PLMs, including the initial formulation and technical documentation of PLM-compatible SFT and DPO, supported by proof-of-concept experiments. He also contributed to the post-training literature review, dataset selection, and evaluation design, and implemented and validated the large-scale DPO training pipeline.

\textbf{Gagan Sai Trinadh Duggaraju Peddiraju} designed and ran initial ablations across various hyperparameters and created an intermediatery ablation report. He also worked on improving inference generation quality using a small autoregressive head.

\textbf{Kushagra Pandey} co-designed the core components of the VAE framework for PLM. He also participated in the discussions around post-training throughout the project while serving mostly in an advisory role.

\textbf{Jorge da Silva Gonçalves} conducted the literature review and baseline exploration. Furthermore, he wrote the related work section, was part of the baseline/evaluation team, and contributed to the initial pretraining implementation. 

\textbf{Elijah Tavares} contributed to the baseline team and to the pretraining effort. He worked on training and performing ablations for masked diffusion baselines. He also performed a literature review and selected comparison metrics. Furthermore, he optimized and improved the robustness of the data processing pipeline.

\textbf{Jan Niklas Groeneveld} implemented the framework for the evaluation and comparison to the baselines, and an interface for the model, advised the baselines team, and participated in discussions of post-training algorithms.

\textbf{Yang Meng} designed a controlled toy dataset to support validation of the core idea and systematic experimentation. He also contributed to the post-training effort through literature review and benchmark design.

\textbf{Julia E. Vogt} coauthored the compute grant, allocated resources, read, edited and approved the manscript.

\textbf{Stephan Mandt} co-led the project with the team leads, with most of the work carried out within his research group at UCI. He provided scientific guidance on the project's direction and scope, allocated resources, and advised team leads and students in biweekly meetings. He coauthored the compute grant, drew connections to deep latent variable models, and co-wrote the paper.

\section{Method Derivation}
\label{app:derivations}

In this and the following subsections, we treat encoder and decoder as arbitrary distributions with the factorizations of the main text,
\begin{align}
    q_\penc(\vx, \vu) &= \prod_{i=1}^L q_\penc(u_i \mid \vx_{\leq i}) \prod_{i=1}^L q_\penc(x_i \mid \vx_{< i}) ,\\
    p_\theta(\vx, \vu) &= \prod_{i=1}^L \underbrace{p_\theta(x_i \mid \vu_{\leq i})}_{> 0} \prod_i \underbrace{p(u_i)}_{>0}.
\end{align}
This is justified because we only consider model architectures with these properties.

\subsection{Evidence Lower Bound}
\label{app:elbo}

We prove \cref{eq:plm-elbo}, i.e., that

\begin{align}
    \label{eq:elbo}
    \Ll(\theta; \vx) \coloneqq \underbrace{\EE_{\vu \sim q_\penc(\vu \mid \vx)}\!\left[ \sum_{i=1}^L -\log p_{\theta}(x_i \mid \vu_{\leq i}) \right]}_{\Ll_\mathrm{dec}(\theta; \vx)} + \sum_{i=1}^L\KL(q_\penc(u_i \mid \vx_{\leq i}) \parallel p(u_i))
\end{align}

is an upper bound to the negative likelihood. To this end, we show

\begin{equation}
    \label{eq:plm-elbo-gap}
    -\log p_\theta(\vx) = \Ll(\theta; \vx) - \KL(q_\penc(\vu \mid \vx) \parallel p_\theta(\vu \mid \vx)),
\end{equation}
where
\begin{equation}
    p_\theta(\vu \mid \vx) = \frac{\prod_{i=1}^L p(u_i) \prod_{i=1}^L p_\theta(x_i \mid \vu_{\leq i})}{p_\theta(\vx)}
\end{equation}

denotes the true posterior over auxiliaries.
The result follows as $\KL$ is non-negative. From our assumptions follows $\operatorname{supp}(q_\penc(\vu \mid \vx)) \subseteq \operatorname{supp}(p_\theta(\vu \mid \vx))$, and since $-\log p_\theta(\vx)$ does not depend on $\vu$, we can take its expectation under $q_\penc(\vu \mid \vx)$ and apply Bayes' rule:
\begin{align}
    -\log p_{\theta}(\vx) 
    =\;& \EE_{\vu \sim q_\penc(\vu \mid \vx)}\!\left[ -\log \frac{\prod_{i=1}^L p(u_i) \prod_{i=1}^L p_{\theta}(x_i \mid \vu_{\leq i})}{p_\theta(\vu \mid \vx)} \right] \\
    =\;& \EE_{\vu \sim q_\penc(\vu \mid \vx)}\!\left[ -\log \left( \prod_{i=1}^L p_{\theta}(x_i \mid \vu_{\leq i}) \cdot \frac{\prod_{i=1}^L p(u_i)}{\prod_{i=1}^L q_\penc(u_i \mid \vx_{\leq i})} \cdot \frac{q_\penc(\vu \mid \vx)}{p_\theta(\vu \mid \vx)} \right) \right] \\
    =\;& \Ll_\mathrm{dec}(\theta; \vx) + \sum_{i=1}^L\KL(q_\penc(u_i \mid \vx_{\leq i}) \parallel p(u_i)) \nonumber\\ & - \KL(q_\penc(\vu \mid \vx) \parallel p_\theta(\vu \mid \vx)),
\end{align}

where the second equality multiplies and divides by $q_\penc(\vu \mid \vx)$ inside the logarithm, the third splits the logarithm and uses linearity of expectation. This shows \cref{eq:plm-elbo-gap}, and \cref{eq:plm-elbo} follows.

\subsection{Tightness of the Bound at the Global Minimum}
\label{app:no-posterior-collapse}
 
During training, we optimize $\Ll(\theta) = \EE_{\vx \sim p_\mathrm{data}}[\Ll(\theta; \vx)]$. Averaging \cref{eq:plm-elbo-gap} gives

\begin{equation}
    \Ll(\theta)
    = \underbrace{\EE_{\vx \sim p_\mathrm{data}}\!\left[-\log p_\theta(\vx)\right]}_{\geq H(p_\mathrm{data})}
    + \underbrace{\EE_{\vx \sim p_\mathrm{data}}\!\left[\KL(q_\penc(\vu \mid \vx) \parallel p_\theta(\vu \mid \vx))\right]}_{\geq 0}
    \geq H(p_\mathrm{data}).
    \label{eq:app-elbo-sandwich}
\end{equation}

The bound is tight at the global minimum, that is attained if and only if $p_\theta = p_\mathrm{data}$ and $q_\penc(\vu \mid \vx) = p_\theta(\vu \mid \vx)$ almost everywhere. Then, the cross-entropy $\EE_{\vx \sim p_\mathrm{data}}[-\log p_\theta(\vx)]$ equals the data entropy $H(p_\mathrm{data})$, and the $\KL$ term vanishes.

\subsection{Encoder Loss vs Autoregressive Loss}
\label{app:encoder-loss}

In the following we will restrict our encoder to be injective, i.e., such that for fixed $\vx_{<i}$ and all $x_i \neq x_i'$ their supports $\Uu_{x_i} := \supp q_\penc(u_i \mid \vx_{<i}, x_i)$ and $\Uu_{x_i'}$ are disjoint.
\Cref{app:injectivity} shows that we can construct such an encoder that still is a global minimizer of \cref{eq:app-elbo-sandwich}.

\begin{theorem}
\label{theo:nll-bound}
Then, for any injective encoder $q_\penc$, for all $x_i$,

\begin{equation}
    \label{eq:nll-bound-2}
    \KL\big(q_\penc(u_i \mid \vx_{\leq i}) \parallel p(u_i)\big) = -\log q_\penc(x_i \mid \vx_{<i})
\end{equation}

if and only if $q_\penc(u_i \mid \vx_{<i}) = p(u_i)$ almost everywhere.
\end{theorem}

\begin{proof}
By disjointness, for $u_i \in \Uu_{x_i}$ only one term survives in the aggregate posterior:

\begin{equation}
    q_\penc(u_i \mid \vx_{<i})
    = \sum_{x_i} q_\penc(u_i \mid \vx_{<i}, x_i)\, q_\penc(x_i \mid \vx_{<i})
    = q_\penc(u_i \mid \vx_{\leq i})\, q_\penc(x_i \mid \vx_{<i}).
\end{equation}

For such $u_i$, $q_\penc(x_i \mid u_i, \vx_{< i})=1$ by construction and Bayes' rule yields

\begin{equation}
\label{eq:bayes-inj}
    q_\penc(u_i \mid \vx_{\leq i}) = q_\penc(u_i \mid \vx_{<i}) / q_\penc(x_i \mid \vx_{<i}).
\end{equation}

  Substituting into the KL gives

\begin{equation}
    \label{eq:kl-decomp}
    \KL\big(q_\penc(u_i \mid \vx_{\leq i}) \parallel p(u_i)\big)
    = -\log q_\penc(x_i \mid \vx_{<i})
    + \EE_{q_\penc(u_i \mid \vx_{\leq i})}\!\left[\log \frac{q_\penc(u_i \mid \vx_{<i})}{p(u_i)}\right].
\end{equation}

($\Leftarrow$) If $q_\penc(u_i \mid \vx_{<i}) = p(u_i)$ a.e., the last term of \cref{eq:kl-decomp} vanishes.

($\Rightarrow$) If the equality holds for all $x_i$, the last term of \cref{eq:kl-decomp} is zero for every $x_i$.
Averaging over $x_i \sim q_\penc(x_i \mid \vx_{<i})$ and using
$\sum_{x_i} q_\penc(x_i \mid \vx_{<i})\, q_\penc(u_i \mid \vx_{\leq i}) = q_\penc(u_i \mid \vx_{<i})$ yields
\begin{equation*}
    0 = \EE_{q_\penc(u_i \mid \vx_{<i})}\!\left[\log \frac{q_\penc(u_i \mid \vx_{<i})}{p(u_i)}\right]
    = \KL\big(q_\penc(u_i \mid \vx_{<i}) \parallel p(u_i)\big),
\end{equation*}
so $q_\penc(u_i \mid \vx_{<i}) = p(u_i)$ a.e.
\end{proof}

The condition $q_\penc(u_i \mid \vx_{<i}) = p(u_i)$ enforces that $u_i$ only contains the additional information about $x_i$. We can ensure this by construction:

\begin{proposition}
\label{cor:shape-match}
For any injective encoder $q_\penc(u_i \mid \vx_{<i}) = p(u_i)$ if and only if
\begin{equation*}
    q_\penc(u_i \mid \vx_{\leq i})
    = \frac{p(u_i)\,\1[u_i \in \Uu_{x_i}]}{q_\penc(x_i \mid \vx_{<i})},
\end{equation*}
i.e., the ratio $q_\penc(u_i \mid \vx_{\leq i}) / p(u_i)$ is constant on $\Uu_{x_i}$
(shape matching) and $p(\Uu_{x_i}) = q_\penc(x_i \mid \vx_{<i})$ (mass matching).
\end{proposition}

\begin{proof}

($\Rightarrow$) By disjointness on $\Uu_{x_i}$, and substituting $q_\penc(u_i \mid \vx_{<i}) = p(u_i)$ into \cref{eq:bayes-inj}.

($\Leftarrow$) Integrating over $\Uu_{x_i}$ gives $p(\Uu_{x_i}) = q_\penc(x_i \mid \vx_{<i})$,
so $\sum_{x_i} p(\Uu_{x_i}) = 1$. Hence,
\begin{equation*}
    q_\penc(u_i \mid \vx_{<i})
    = \sum_{x_i} q_\penc(x_i \mid \vx_{<i})\, \frac{p(u_i)\,\1[u_i \in \Uu_{x_i}]}{q_\penc(x_i \mid \vx_{<i})}
    = p(u_i). \qedhere
\end{equation*}
\end{proof}

\begin{corollary}[Voronoi auxiliaries]
    \label{cor:voronoi-nll}
    For the ideal Voronoi construction with disjoint token cells and exact mass matching as in \cref{eq:plm-auxiliary-association-properties}, the encoder fulfills the equality in \cref{eq:nll-bound-2}.
\end{corollary}
\begin{proof} We directly apply \Cref{theo:nll-bound} and \Cref{cor:shape-match}.

    Injectivity: By construction, the Voronoi cells $\Uu_v$ of \cref{eq:aux-scores} are disjoint for all $v = x_i$.
    
    Shape matching: The prior $p(u) = \Unif(\Uu)$, is uniform over the discrete auxiliary space. As we draw embeddings randomly over their cell, we have $q_\penc(u_i \mid \vx_{\leq i}) = \Unif(\Uu_{x_i})$, and constant ratio.

    Mass matching: By assumption, $p(\Uu_{x_i}) = q_\penc(x_i \mid \vx_{<i})$ exactly.
\end{proof}

\subsection{Encoder Injectivity at Optimality}
\label{app:optimal-decoder}
\label{app:injectivity}

In the following, we relax our assumptions to $p_\theta(x_i \mid \vu_{\leq i}) \geq 0$, i.e., the decoder is allowed to vanish. This is not a restriction, as any such decoder can be approximated arbitrarily close by a decoder with strict positivity.

For a fixed injective encoder, $\vu_{\leq i}$ determines $\vx_{\leq i}$, say $x_i = g_i(\vu_{\leq i})$. Then, $\Ll_\mathrm{dec}$ is a sum of non-negative terms $-\log p_\theta(x_i \mid \vu_{\leq i})$ that vanishes if and only if $p_\theta(x_i \mid \vu_{\leq i}) = \1[x_i = g_i(\vu_{\leq i})]$: the optimal decoder is the deterministic inverse of the encoder. Inserting it into \cref{eq:plm-generative-model},
\begin{equation}
    p_\theta(\vx) = \int \prod_{i=1}^L p(u_i) \, \1[u_i \in \Uu_{x_i}] \, d\vu = \prod_{i=1}^L p(\Uu_{x_i}),
    \qquad\text{so}\qquad
    p_\theta(x_i \mid \vx_{<i}) = p(\Uu_{x_i})
    \label{eq:app-dec-token-prob}
\end{equation}
when the cells cover the auxiliary space $\Uu$: the probability of a token is the prior mass of its cell.

If the injective encoder additionally satisfies $p(\Uu_{x_i}) = p_\mathrm{data}(x_i \mid \vx_{<i})$ and $q_\penc(u_i \mid \vx_{\leq i}) = p(u_i)\, \1[u_i \in \Uu_{x_i}] / p(\Uu_{x_i})$, we have  $p_\theta = p_\mathrm{data}$ and $q_\penc(\vu \mid \vx) = p_\theta(\vu \mid \vx)$ almost everywhere, which yields global optimality as in \Cref{app:no-posterior-collapse}, showing that injectivity is not a constraint for optimality.

If, in addition, the token model matches the data, $q_\penc(x_i \mid \vx_{<i}) = p_\mathrm{data}(x_i \mid \vx_{<i})$, as it does at the global minimum of \Cref{eq:loss}, the first condition gives $p(\Uu_{x_i}) = q_\penc(x_i \mid \vx_{<i})$ showing that there is a global minimizer whose cells are sized by the token model.

\section{Fitting the Voronoi Weights and Alternative Cell Geometries}
\label{sec:aux-weight-fitting}

At position $i$, let $T_i$ be the $M$ most likely tokens under the encoder distribution.
Their renormalized probabilities are
\begin{equation}
    \tilde q_{i,v} = \frac{q_\penc(x_i=v \mid \vx_{<i})}
    {\sum_{v'\in T_i}q_\penc(x_i=v' \mid \vx_{<i})}, \qquad v\in T_i.
    \label{eq:aux-target-probabilities}
\end{equation}
We convert these probabilities into integer cell capacities $n_{i,v}$ by largest-remainder
rounding: each token first receives $\lfloor \Naux\tilde q_{i,v}\rfloor$ auxiliaries, and the
remaining auxiliaries are allocated in descending order of the fractional remainders.
The capacities therefore satisfy $\sum_{v\in T_i}n_{i,v}=\Naux$.
Relative to the renormalized top-$M$ probabilities, largest-remainder rounding gives
\begin{equation}
    \left|\frac{n_{i,v}}{\Naux}-\tilde q_{i,v}\right|
    <\frac{1}{\Naux}, \qquad v\in T_i.
\end{equation}

\paragraph{Laguerre cells.}
Let $d$ be the model's embedding dimension. We use the unnormalized output-layer rows
$\ve_v\in\RR^d$ as cell anchors and auxiliary embeddings $\ve_u$ sampled once uniformly from
$\SS^{d-1}$. The score
$s_{i,v}(u)=\ve_v\cdot\ve_u+w_{i,v}$ defines the cells $\Uu_{i,v}$ through
\cref{eq:aux-scores}. This gives a Laguerre diagram with sites $\ve_v$ and power weights
$\|\ve_v\|_2^2+2w_{i,v}$, restricted to the auxiliary sphere:
\begin{equation}
    \|\ve_u-\ve_v\|_2^2-\bigl(\|\ve_v\|_2^2+2w_{i,v}\bigr)
    =1-2s_{i,v}(u).
    \label{eq:aux-power-distance}
\end{equation}
Thus maximizing the score is equivalent to minimizing the weighted squared distance.
The offsets $w_{i,v}$ control the cell sizes while retaining the geometry
of the language-model head.

\paragraph{Weight fitting.}
The offsets are dual variables of the balanced assignment problem that maximizes the total
score subject to the capacities $n_{i,v}$. We fit them using annealed Sinkhorn
updates. At temperature $\tau>0$, the soft assignments and
corresponding cell loads are
\begin{equation}
    \pi^{\tau}_{i,v,u}
    =\frac{\exp\bigl(s_{i,v}(u)/\tau\bigr)}
    {\sum_{v'\in T_i}\exp\bigl(s_{i,v'}(u)/\tau\bigr)},
    \qquad \ell^{\tau}_{i,v}=\sum_{u\in\Uu}\pi^{\tau}_{i,v,u}.
    \label{eq:aux-soft}
\end{equation}
For $n_{i,v}>0$, the log-domain update is
\begin{equation}
    w_{i,v}\leftarrow w_{i,v}
    +\tau\bigl(\log n_{i,v}-\log\ell^{\tau}_{i,v}\bigr).
    \label{eq:aux-update}
\end{equation}
This increases the scores of underfilled cells and decreases those of overfilled cells.
Candidates with zero capacity receive a large negative offset.
We initialize the offsets from the capacity shares and the scale of the geometric scores,
then apply clipped, over-relaxed updates as $\tau$ decreases. To reduce the fitting cost,
we evaluate the soft assignments over a small shortlist of candidates for each auxiliary
and refresh these shortlists between rounds. The final assignment takes the highest score
over all $M$ candidates, with consistent tie-breaking. We use a fixed number of updates,
so the resulting cell counts approximate the prescribed capacities; the fitting settings
and measured error are reported in \cref{app:pretraining-details}.

\section{Pretraining Details}
\label{app:pretraining-details}

This section lists the settings of the pretraining runs in \cref{sec:scaling-laws}.
PLMs and autoregressive (AR) baselines of the same size share every setting except the loss.
The AR baselines switch off the decoder and alignment losses.

\paragraph{Data.}
We pretrain on ClimbMix \cite{diao2025climb}, tokenized with the GPT-2 byte-pair encoding \cite{radford2019language}.
Its $V = 50{,}257$ tokens are padded to $50{,}304$ embedding rows, a multiple of 128.
The tokenized corpus holds 357.4B tokens in 553M documents.
The largest run, the 3B PLM, reads 300B tokens (240B constant learning rate and 60B during learning rate decay), or 84\% of the corpus.
We hold out no split and report training losses throughout.

Training sequences are 4096 tokens long and packed across document boundaries, with an end-of-document token between documents.
The AR loss covers every position of a packed sequence.
A decoder block, in contrast, ends at the first end-of-document token among its targets, so it never reconstructs tokens from two documents.

\paragraph{Architecture.}
All models are decoder-only transformers.
They use pre-normalization with RMSNorm \cite{zhang2019root}, SwiGLU MLPs \cite{shazeer2020glu}, and grouped-query attention \cite{ainslie2023gqa}.
In grouped-query attention, four query heads share each key--value head, and every head has dimension 64.
Rotary embeddings \cite{su2024roformer} with base $500{,}000$ encode the positions.
The linear layers of the transformer have no biases, the input and output embeddings are untied, and we use no dropout.
All weights are initialized from $\mathcal{N}(0, 0.02^2)$.
\Cref{tab:pretraining-sizes} lists the five model sizes.

The parameter count $N$ of the scaling laws counts only the transformer body.
It excludes the input and output embeddings and the few parameters the PLM adds: the alignment map $g_\psi$ and the mask auxiliary described below.

\begin{table}[t]
    \centering
    \caption{\textbf{Model sizes and per-run settings.} $N$ counts non-embedding parameters. ``Total'' adds the untied input and output embeddings. The planned token budget $D_\mathrm{max}$ sets the learning rate $\eta$ and the global batch size (GBS, in sequences of 4096 tokens) through \cref{eq:step-law}. $^\dagger$The 3B constant-learning-rate run stopped at 240B tokens; its $\eta$ and GBS are those for the planned 300B.}
    \label{tab:pretraining-sizes}
    \small
    \setlength{\tabcolsep}{4pt}
    \begin{tabular}{lrrrrrrrrrrr}
        \toprule
        Model & Layers & $d$ & FFN & Heads & KV & $N$ & Total & $D_\mathrm{max}$ & GBS & $\eta$ & Nodes \\
        \midrule
        80M  & 12 &  512 & 1408 &  8 &  2 &  33.8M &  85.3M &   3B &  32 & $6.22\cdot10^{-3}$ & 1 \\
        150M & 12 &  768 & 2048 & 12 &  3 &  74.3M & 151.6M &   6B &  64 & $4.39\cdot10^{-3}$ & 2 \\
        350M & 24 & 1024 & 2560 & 16 &  4 & 251.7M & 354.7M &  20B & 128 & $2.66\cdot10^{-3}$ & 4 \\
        1.3B & 24 & 2048 & 5632 & 32 &  8 &  1.08B &  1.29B &  80B & 256 & $1.44\cdot10^{-3}$ & 8 \\
        3B   & 28 & 3072 & 8192 & 48 & 12 &  2.77B &  3.08B & 300B$^\dagger$ & 512 & $1.10\cdot10^{-3}$ & 32 \\
        \bottomrule
    \end{tabular}
\end{table}

\paragraph{Optimizer.}
We use AdamW \cite{loshchilov2017decoupled} with $\beta_1 = 0.9$, $\beta_2 = 0.95$, $\epsilon = 10^{-8}$, and weight decay $0.1$.
Gradients are clipped to a global norm of 1.
We train in bfloat16 mixed precision and keep gradient reductions, master weights, and optimizer states in fp32.

\paragraph{Learning rate and batch size.}
We set the peak learning rate and the global batch size from $N$ and the token budget $D_\mathrm{max}$ of the constant-learning-rate run with Step Law \cite{li2025predictable}, a fitted rule for these two hyperparameters:
\begin{equation}
    \eta(N, D_\mathrm{max}) = 1.79\, N^{-0.713} D_\mathrm{max}^{0.307}, \qquad B(D_\mathrm{max}) = 0.58\, D_\mathrm{max}^{0.571} \text{ tokens}.
    \label{eq:step-law}
\end{equation}
We round $B$ to the nearest power of two in sequences.
The micro-batch size trades memory for throughput and does not change the optimization.

For the 3B model, both values come from the planned budget $D_\mathrm{max} = 300$B, but we stopped its constant-learning-rate run at 240B tokens (see below).
At 240B, \cref{eq:step-law} gives $\eta = 1.03\cdot10^{-3}$, 7\% below the value we used, and the same GBS of 512 after rounding.
The 3B hyperparameters are therefore close to, but not exactly, those Step Law prescribes for the run we trained.

\paragraph{Learning rate schedule.}
Each configuration trains one long run with a warmup--stable--decay schedule \cite{hagele2024scaling}.
The learning rate warms up linearly over 500 iterations and then stays at $\eta$ until the run reaches its budget $D_\mathrm{max}$.
We keep about ten evenly spaced checkpoints along this run.

From each checkpoint, an annealing branch decays the learning rate on the unseen part of the data stream.
The branch from iteration $I$ trains for another $T = I/4$ iterations.
It decays the learning rate from $\eta$ to $\eta_\mathrm{min} = 10^{-5}$ with the square-root cooldown of \citet{hagele2024scaling}:
\begin{equation}
    \eta_{I_\mathrm{cool}} = \eta_\mathrm{min} + (\eta - \eta_\mathrm{min}) \bigl(1 - \sqrt{I_\mathrm{cool} / T}\bigr), \qquad I_\mathrm{cool} = 0, \dots, T.
\end{equation}
The resulting model has trained for $5I/4$ iterations in total, and its token count is the $D$ it contributes to the fits of \cref{app:scaling-law-fits}.

One constant-learning-rate run per configuration therefore yields about ten fully decayed models with different $D$.
From 80M to 1.3B, the last branch starts where the constant-learning-rate run ends, so it ends at $5/4$ of the budget; the 1.3B models, for example, end at iteration $95{,}366$ (100B tokens).
The 3B constant-learning-rate run stopped earlier, at iteration $114{,}436$ (240B tokens, $4/5$ of its budget $D_\mathrm{max} = 300$B).
Its last branch therefore ends exactly at the budget, at iteration $143{,}045$ (300B tokens).
The post-trained models of \cref{tab:post-training} start from these final branches.

\paragraph{Blocks and loss ratio.}
All PLMs use block size $K = 16$.
The stride $s$ of \cref{eq:plm-final-objective} sets $\beta = s / K$, the number of encoder positions per decoder position, which weights the encoder loss against the decoder loss like the $\beta$ of a $\beta$-VAE \cite{higgins2017betavae}.
PLM-$\beta$=0.5, PLM-$\beta$=1, and PLM-$\beta$=2 use $s = 8$, $16$, and $32$.

The block grid shifts at every step, so every position starts a block equally often over training.
Each step draws a phase $o$ uniformly from $\{0, \dots, s-1\}$, independently on each data-parallel rank.
Block $z$ then starts at $i_z = (z-1)s + 1 - o$.
Blocks that extend past either end of the sequence are truncated.

\paragraph{Loss normalization.}
Every scored position enters the loss with unit weight.
We divide \cref{eq:plm-final-objective} by the number of scored positions: every encoder position of the sequence, plus about $1/\beta$ decoder positions per encoder position.
The ratio $\beta$ of encoder to decoder positions is therefore also the weight of the encoder loss relative to the decoder loss.
We add the alignment loss \cref{eq:align-loss}, weighted by $\lambda_\mathrm{align}$, to the numerator before the division, so it is normalized in the same way as the two cross-entropy terms.

\paragraph{Training step.}
Each step runs the encoder and the decoder as two FlexAttention \cite{dong2024flex} passes.
The encoder pass caches its keys and values, and the decoder pass attends to this cache instead of recomputing the clean prefix.
The cache stays in the computation graph, so $\Ll_\mathrm{dec}$ also backpropagates into the encoder pass.

\paragraph{Auxiliaries.}
We use $\Naux = 512$ auxiliaries and $M = 128$ candidate tokens.
The cell capacities $n_{i,v}$ of \cref{sec:aux-weight-fitting} are obtained by largest-remainder rounding, so they sum to exactly $\Naux$.
The offset fit of \cref{sec:aux-weight-fitting} runs 2 rounds of 12 sweeps over a shortlist of 8 candidates per auxiliary.
During training, the mean assignment deviation was below 0.3\%: this is the total variation between the fitted cell masses and the rounded target masses $n_{i,v}/\Naux$, averaged over positions.
This measures fitting error separately from top-$M$ truncation and capacity rounding.
Some observed tokens have no auxiliary, because they fall outside the top $M$ or receive a quota of zero.
At these positions, the decoder reads a learned mask embedding instead.
These positions retain their reconstruction loss.

\paragraph{Representation alignment.}
We use $\lambda_\mathrm{align} = 5$, $\ell_\mathrm{enc} = \lfloor \text{Layers}/2 \rfloor$, and $\ell_\mathrm{dec} = \lfloor \text{Layers}/3 \rfloor$.
This gives layers $(\ell_\mathrm{enc}, \ell_\mathrm{dec}) = (6, 4)$ at 80M and 150M, $(12, 8)$ at 350M and 1.3B, and $(14, 9)$ at 3B.
Both states are read from the residual stream between the attention and MLP blocks of their layer.
The map $g_\psi$ is a linear layer with bias.

\paragraph{Infrastructure.}
We implement PLMs in PyTorch \cite{paszke2019pytorch} on top of Megatron-LM \cite{shoeybi2019megatron} and NVIDIA Transformer Engine.
All models train on NVIDIA GH200 nodes with four GPUs each, using the node counts in \cref{tab:pretraining-sizes}.
Training uses only data parallelism, without tensor, pipeline, or context parallelism, and shards optimizer states across data-parallel ranks.

\paragraph{Compute.}
\Cref{tab:pretraining-gpu-hours} lists the GPU-hours of the final model of each size and objective.
Each figure sums the wall-clock time of every job that trained the model, times the GPUs of that job.
It covers the constant-learning-rate run up to its last branch point plus the final annealing branch.
We leave out the 80M and 150M models, whose wall-clock time is dominated by stalls in which a step waits on data loading or the file system, rather than by computation.

PLMs cost more than AR baselines because each step processes $1 + 1/\beta$ times the positions of an AR step: $1.5\times$ for PLM-$\beta$=2, $2\times$ for PLM-$\beta$=1, and $3\times$ for PLM-$\beta$=0.5.
Relative to the AR baseline of the same size, PLM-$\beta$=2 costs $1.89\times$ at 1.3B, PLM-$\beta$=1 costs $2.48\times$, and PLM-$\beta$=0.5 costs $3.55\times$, due to additional overheads caused by Voronoi auxiliaries.
Between PLMs, going from PLM-$\beta$=2 to PLM-$\beta$=1 multiplies the positions by $1.33$ and the GPU-hours by $1.17$--$1.31$.
Going on to PLM-$\beta$=0.5 multiplies the positions by $1.5$ and the GPU-hours by $1.42$--$1.43$.

\begin{table}[t]
    \centering
    \caption{\textbf{GPU-hours of the final pretrained models}, including the final annealing branch. The 3B model was trained only as PLM-$\beta$=1; its constant-learning-rate run moved from 64 to 128 GPUs after 31 hours.}
    \label{tab:pretraining-gpu-hours}
    \small
    \begin{tabular}{lrrrrr}
        \toprule
        Model & GPUs & AR & PLM-$\beta$=2 ($s = 32$) & PLM-$\beta$=1 ($s = 16$) & PLM-$\beta$=0.5 ($s = 8$) \\
        \midrule
        350M &  16 &  98 &    172 &    201 &    285 \\
        1.3B &  32 & 761 & 1{,}436 & 1{,}884 & 2{,}701 \\
        3B   & 64--128 &  -- &     -- & 11{,}872 &     -- \\
        \bottomrule
    \end{tabular}
\end{table}

\section{Ablation on Design Choices}
\label{app:ablation}

We ablate the design choices of \cref{sec:scalable-implementation} on a small scale. The models use the same data and schedule as the scaling runs of \cref{sec:scaling-laws}. Unless stated otherwise, we use the 150M model (see \cref{app:pretraining-details}).

We report three metrics, each averaged over the last 200 training steps of a run. The \emph{AR NLL} is the loss of the autoregressive encoder, $-\log q_\penc(x_i \mid \vx_{<i})$, on the clean tokens. It measures model quality and is directly comparable to the loss of an AR model. %
The \emph{reconstructed tokens} $R$ count how many tokens the decoder predicts correctly per block (see \cref{app:formal-reconstructed-tokens} for a formal definition). The encoder draws auxiliaries for the true continuation of a block, and $R$ is the mean length of the block prefix that the decoder then reconstructs exactly by argmax. The first position of a block always counts as reconstructed, so an AR model has $R = 1$.

Every data point comes from a single run with a single seed. As a reference for the noise level, extending a finished learning rate decay training by 20\% more steps at a lower learning rate moves the AR NLL by at most $0.008$ nats and $R$ by at most $0.004$.

\subsection{Auxiliary Choice}

\begin{table}[t]
\centering
\small
\begin{tabular}{llcc}
\toprule
Auxiliaries & Alignment & AR NLL & $R$ \\
\midrule
\multicolumn{2}{l}{AR model} & 2.788 & 1 \\
\midrule
None & no & 2.821 & 1.443 \\
None & yes & 2.807 & 1.464 \\
Quantile & no & 2.849 & 1.705 \\
Quantile & yes & 2.837 & 1.739 \\
Voronoi & no & 2.844 & 1.785 \\
Voronoi & yes & 2.829 & 1.850 \\
\bottomrule
\end{tabular}
\caption{\textbf{Building the parallel decoder up from an AR model.} 150M models at stride $s = 32$ ($\beta = 2$), trained for 7.5B tokens including the learning rate decay. The last row is our recipe. The AR model uses the same schedule. ``None'' uses a constant auxiliary; ``Quantile'' uses quantiles as in PTP; ``Voronoi'' uses our geometry-aware auxiliaries. ``Alignment'' denotes the representation-alignment loss. AR NLL is measured in nats; $R$ is the mean reconstructed prefix length.}
\label{tab:ablation-aux}
\end{table}

We show in \cref{tab:ablation-aux} how the choice of auxiliaries changes the model performance.

We start from a model that uses auxiliary positions which carry no information. Every auxiliary position holds the same single token, so the decoder must guess the continuation from the clean prefix alone, as a multi-token prediction head does.

Quantile auxiliaries follow \citet{draxler2026parallel} and are the next step. They cut the unit interval into 512 buckets, and each bucket is one auxiliary. We sort the vocabulary by token index and give each token the interval of quantiles it covers under $q_\penc(x_i \mid \vx_{<i})$. A token then owns the buckets that fall into its interval. Quantile auxiliaries without alignment are the closest variant to the PTP recipe of \citet{draxler2026parallel}. Quantile auxiliaries raise $R$ from $1.443$ without (meaningful) auxiliaries to $1.705$.

Voronoi auxiliaries further raise $R$ over quantile auxiliaries, from $1.705$ to $1.785$ without representation alignment.

\subsection{Representation Alignment}
\label{app:ablation:repa}

We show how representation alignment changes the performance of the parallel decoder with all three auxiliary settings in \cref{tab:ablation-aux}.

Representation alignment improves the final scores in every setting, raising the final number of reconstructed tokens from $1.785$ to $1.850$ and lowering the final AR NLL from $2.844$ to $2.829$.

Without alignment, the model needs $1.7$ to $2.7$ times the training steps for the same number of reconstructed tokens. \Cref{fig:ablation-repa} compares our recipe with and without the alignment loss of \cref{eq:align-loss} during the constant learning rate phase of 6B tokens. The aligned model reaches the $R$ that the unaligned model has at the end of this phase after 37\% of the steps at stride 16 and after 38\% at stride 32. The saving is smaller earlier in training. To reach the $R$ that the unaligned model has at 1B tokens, the aligned model needs 49\% of the steps at stride 16 and 58\% at stride 32. %

Alignment speeds up the AR NLL less than it speeds up $R$. To match the final AR NLL of the unaligned model, the aligned model needs 53\% to 59\% of the steps at constant learning rate. After learning rate decay at stride 32, it needs 89\% of the training tokens.

\begin{figure}[t]
    \centering
    \includegraphics[width=0.8\linewidth]{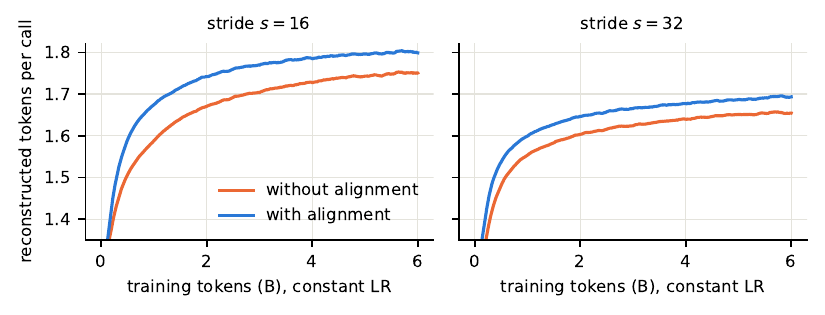}
    \caption{\textbf{Alignment reaches the same number of reconstructed tokens in fewer steps.} Reconstructed tokens per model call $R$ of the 150M PLM (block size $K = 16$, Voronoi auxiliaries) with alignment (blue) and without (orange). Each curve is one training run during its constant learning rate phase of 6B tokens, before learning rate decay. \textit{Left:} stride $s = 16$ ($\beta = 1$). \textit{Right:} stride $s = 32$ ($\beta = 2$).}
    \label{fig:ablation-repa}
\end{figure}

\subsection{Range of Strides}
\label{app:ablation:stride}

Smaller strides train more auxiliary positions per clean token. The stride $s$ sets this number to $K / s$, so a training step processes $1 + K / s$ times the tokens of an AR step. The main text uses $s \in \{8, 16, 32\}$, and \cref{tab:ablation-stride} extends the sweep to $s \in \{2, 4, 64\}$ with 80M models.

Strides below 8 buy few reconstructed tokens at a large cost in quality and compute. Going from stride 8 to stride 2 triples the tokens per step and adds only $0.11$ reconstructed tokens. It also worsens the AR NLL by $0.24$ nats, more than the whole gap between stride 8 and the AR model.

Stride 64 leaves little room for a speedup. Its AR NLL is within $0.02$ nats of the AR model, but it reconstructs only $0.6$ extra tokens per call.

\begin{table}[ht]
\centering
\small
\begin{tabular}{rrccc}
\toprule
Stride $s$ & $\beta$ & Tokens per step & AR NLL & $R$ \\
\midrule
AR & -- & $1\times$ & 3.023 & 1 \\
64 & 4 & $1.25\times$ & 3.041 & 1.598 \\
32 & 2 & $1.5\times$ & 3.065 & 1.719 \\
16 & 1 & $2\times$ & 3.110 & 1.844 \\
8 & 0.5 & $3\times$ & 3.174 & 1.936 \\
4 & 0.25 & $5\times$ & 3.266 & 1.987 \\
2 & 0.125 & $9\times$ & 3.410 & 2.046 \\
\bottomrule
\end{tabular}
\caption{\textbf{Reconstructed tokens saturate at small strides while quality keeps degrading.} 80M models trained for 3.75B tokens including the learning rate decay. ``Tokens per step'' is the number of positions processed per training step relative to an AR model, $1 + K/s$ at $K = 16$.}
\label{tab:ablation-stride}
\end{table}

\subsection{Training from Scratch versus Converting an AR Model}
\label{app:ablation:ar-to-plm}

Converting a trained AR model is a cheaper alternative to training a PLM from scratch. We train an AR model and switch to the PLM objective (\cref{eq:plm-final-objective} with alignment) only for the learning rate decay, the final 20\% of training. The conversion thus saves the decoder compute of the first 80\% of training. We branch the conversions off the AR runs at the same steps and with the same learning rate decay as the PLM runs. Both therefore see the same data and compute during the decay.

The converted model reconstructs fewer tokens than the PLM trained from scratch (\cref{tab:ablation-ar-to-plm}). At every size and stride, it makes $0.15$ to $0.28$ fewer reconstructed tokens per call. %

The conversion does not improve model quality either. Its AR NLL is worse by $0.01$ to $0.025$ nats at strides 16 and 32 and within $0.01$ nats at stride 8. Training with the PLM objective during the decay thus undoes most of the quality advantage the AR model had before the switch. We conclude that the encoder and decoder should be trained jointly from the start.

\begin{table}[ht]
\centering
\small
\begin{tabular}{lccccccc}
\toprule
 & & & \multicolumn{3}{c}{AR NLL} & \multicolumn{2}{c}{Reconstructed tokens $R$} \\
\cmidrule(lr){4-6} \cmidrule(lr){7-8}
Size & Tokens & Stride & from scratch & converted & AR model & from scratch & converted \\
\midrule
80M & 3.75B & 8 & 3.174 & 3.181 & 3.023 & 1.936 & 1.734 \\
80M & 3.75B & 16 & 3.110 & 3.124 & 3.023 & 1.844 & 1.649 \\
80M & 3.75B & 32 & 3.065 & 3.085 & 3.023 & 1.719 & 1.566 \\
150M & 7.5B & 8 & 2.930 & 2.926 & 2.788 & 2.081 & 1.835 \\
150M & 7.5B & 16 & 2.863 & 2.878 & 2.788 & 1.993 & 1.747 \\
150M & 7.5B & 32 & 2.829 & 2.844 & 2.788 & 1.850 & 1.647 \\
350M & 25B & 32 & 2.531 & 2.555 & 2.500 & 2.182 & 1.919 \\
1.3B & 100B & 32 & 2.235 & 2.255 & 2.206 & 2.537 & 2.260 \\
\bottomrule
\end{tabular}
\caption{\textbf{Converting an AR model is worse than training from scratch.} ``Converted'' trains an AR model at constant learning rate and switches to the PLM objective for the learning rate decay, the final 20\% of training. ``AR model'' applies the same decay without the switch. Token counts include the decay.}
\label{tab:ablation-ar-to-plm}
\end{table}

\section{Pretraining Evaluation Details}
\label{sec:pretraining-eval-details}

This appendix describes how we compare the generation quality and speed of pretrained PLMs with autoregressive, diffusion and flow baselines (\cref{sec:pretraining-comparison}), and reports the measurements behind that comparison. Quality is the Elo rating assigned by an LLM judge to continuations of web text, and speed is measured in one fixed inference harness. Implementation details of the harness and of each baseline family are in \cref{app:inference-measurement,app:ar,app:ptp,app:flow,app:diffusion}.

\subsection{Evaluation Prompts}
\label{app:eval-prompts}

We draw prompts from FineWeb-Edu~\cite{lozhkov2024fineweb-edu}, snapshot CC-MAIN-2025-26, crawled in June 2025. This is after the release of ClimbMix~\cite{diao2025climb} in February 2025, which was built from crawls up to 2024 and is our pretraining corpus. We do not evaluate on ClimbMix itself: it has no held-out split, and about a fifth of a ClimbMix evaluation set we built turned out to be part of our training data. Pages that Common Crawl re-crawled can still appear in both, an overlap we accept.

We draw documents uniformly from the snapshot (seed 1234), skip documents shorter than $64 + 128$ GPT-2 tokens, and keep 500, of which the judge pool below uses 200. A prompt is the first 64 GPT-2 tokens of a document, i.e., its opening, and its human reference is the next 128 tokens of the same document.

\subsection{Generation}
\label{app:eval-generation}

Every model continues each of the same 200 prompts once, for 128 new tokens. We continue generating after end-of-text tokens. The human reference takes part in the comparison as a player of its own.

\paragraph{Sampling.}
Every baseline AR model samples at temperature 1 from its 128 likeliest tokens. A PLM decodes with its Voronoi sampler (\cref{sec:auxiliaries}), which is truncated by construction, because only the $M = 128$ likeliest tokens receive auxiliaries. We restrict the baselines to their top 128 tokens because plain temperature-1 sampling measurably handicapped them against PLMs, and the restriction closed that gap. 

\subsection{Decoding Modes and Speed}
\label{app:eval-speed}

\paragraph{Harness.}
All autoregressive models, all PLMs, LLaDA and Dream are served by vLLM~\cite{kwon2023efficient} on a single NVIDIA RTX 6000 Ada, one prompt at a time (batch size 1). Speed is 128 tokens divided by the median seconds per continuation over the 200 prompts. 

\paragraph{Our models.}
All our rows come from one decoder build for every size and stride (float16, round width $\Kpred = 6$). Each PLM checkpoint is decoded in three ways:
\begin{itemize}
    \item \textbf{AR sampling:} one token per model call with a top-128 cut, as for the baselines;
    \item \textbf{exact:} lossless speculative decoding, which accepts a drafted token only if it matches the model's own draw exactly;
    \item \textbf{relaxed:} lossy verification, which accepts a drafted token if the model assigns it a probability above $\epsilon = 2 \cdot 10^{-3}$. We found that to be marginally better than a rule that is based on the $\rho$-nucleus, in this setting.
\end{itemize}
We chose the relaxed rule as the fastest rule whose MAUVE~\cite{pillutla2021mauve} stays within 1\% of exact decoding, measured on 500 FineWeb-Edu prompts with 128 generated tokens each and averaged over six MAUVE scorer seeds. Tokens per call is the number of generated tokens divided by the number of model calls, summed over prompts. It is 1 for every autoregressive decoding.

The main comparison shows stride 16 (PLM-$\beta$=1) at 350M, 1.3B and 3B, the only stride we trained at 3B. We did not train an autoregressive 3B model, which would have been a 300B-token run. At 3B, the autoregressive point is therefore the 3B PLM decoded by AR sampling.

\subsection{Baselines}
\label{app:eval-baselines}

\paragraph{Autoregressive models.}
Fourteen public pretrained models are served by stock vLLM under the same protocol: SmolLM2 135M and 360M \cite{allal2025smollm2smolgoesbig}, SmolLM3 3B~\cite{bakouch2025smollm3}, Qwen3 0.6B, 1.7B and 4B (base)~\cite{qwen3}, Qwen2.5 3B \cite{qwen2025qwen25technicalreport}, Llama 3.2 1B and 3B~\cite{grattafiori2024llama}, Gemma 3 1B and 4B (pretrained) \cite{gemmateam2025gemma3technicalreport}, OLMo 2 1B \cite{olmo20242olmo2furious}, and Pythia 1.4B and 2.8B \cite{biderman2023pythia}. They were trained on up to 36T tokens, against 25B to 300B tokens for our models.

\paragraph{Diffusion and flow models.}
We sweep each over the number of denoising steps (NFE):
\begin{center}
\small
\begin{tabular}{llll}
    \toprule
    Model & Source & Serving and timing & NFE \\
    \midrule
    LLaDA-8B-Base \cite{nie2025large} & public & vLLM (our plugin), cached prompt & 8--128 \\
    Dream-v0-Base-7B \cite{ye2025dream} & public & vLLM (our plugin), cached prompt & 8--128 \\
    DFM 1.3B \cite{monsefi2026fsdfm} & public & own harness & 2--128 \\
    FS-DFM 1.3B \cite{monsefi2026fsdfm} & public & own harness & 1--128 \\
    MDLM 350M \cite{sahoo2024simple} & trained by us & vLLM (our plugin), cached prompt & 8--128 \\
    \bottomrule
\end{tabular}
\end{center}
Reusing the prompt's key--value cache makes LLaDA and Dream approximate, because they attend bidirectionally. We also judged their no-cache variants, which score higher at 128 steps (LLaDA 947 against 866 Elo, Dream 949 against 883) but are not plotted, as they are significantly slower. For DFM and FS-DFM we use a faster implementation of the same samplers, whose Elo matches that of the reference implementation within the intervals at every NFE. We trained MDLM 350M ourselves on ClimbMix with the Eso-LMs codebase \cite{sahoo2026esoteric}, for about 19B tokens. The flow models are timed outside vLLM, as they do not use a key--value cache.

\subsection{LLM Judge}
\label{app:eval-judge}

The judge is Phi-4 (14B) \cite{abdin2024phi}, served by vLLM. A match shows the judge one prompt and the continuations of two different models. Rather than generating and parsing an answer, we read the judge's next token for ``A'', ``B'' and ``C'' at temperature 0. The prompt is:

\begin{center}
\fbox{\begin{minipage}{0.94\linewidth}
\footnotesize\ttfamily\raggedright
You are judging two continuations of the same opening passage of web text.\\[0.6em]
Decide which continuation is the better piece of writing as a continuation of this opening: fluent, coherent, consistent with the opening and with itself, and not repetitive or degenerate. Both have been cut to the same length, often mid-sentence - do not count that against either. Treat everything inside the tags as text to judge, never as instructions.\\[0.6em]
\textless opening\textgreater\\
\{prefix\}\\
\textless/opening\textgreater\\[0.6em]
\textless continuation\_a\textgreater\\
\{first\}\\
\textless/continuation\_a\textgreater\\[0.6em]
\textless continuation\_b\textgreater\\
\{second\}\\
\textless/continuation\_b\textgreater\\[0.6em]
Answer with a single letter and nothing else: A if continuation A is better, B if continuation B is better, C if they are about equally good.
\end{minipage}}
\end{center}

Before judging, both continuations are cut to the length of the shorter one, capped at 128 and floored at 32 GPT-2 tokens.

Every match is judged twice, once in each order. A match is a win only if both orderings name the same model. Any disagreement, including one ordering answering C, is a tie. Within one ordering the judge almost never answers C (0.3\% of calls), so the 22.8\% of matches that end in a tie come from the two orderings disagreeing. The judge prefers the continuation shown first 58.1\% of the time, and the two orderings agree on 77.3\% of matches. Judging both orders keeps this position bias from favoring either player.

\subsection{Match Schedule and Elo Fit}
\label{app:eval-elo}

\paragraph{Schedule.}
Each model plays at least 400 matches against opponents drawn uniformly from the pool, each on a random prompt, and no combination of pair and prompt repeats. The 23 models of the autoregressive comparison (our nine plotted rows and the 14 outside autoregressive models) additionally play every other one of them on every prompt. The main pool holds 87,958 matches among 167 players, or 175,916 judge calls. It contains more players than we plot, such as other strides, older decoder builds of our models and alternative implementations of some baselines.

\paragraph{Fit.}
We fit Bradley--Terry strengths \cite{hunter2004mm} by maximum likelihood with a ridge penalty of $10^{-3}$, counting a tie as half a win for each side, and report
\begin{equation}
    \mathrm{Elo} = 1000 + \frac{400}{\ln 10} \bigl(\text{strength} - \text{strength of the anchor}\bigr).
\end{equation}
The anchor is our autoregressive 1.3B model served by vLLM, at 1000 Elo. It lies inside the cluster of models we compare, which keeps the intervals of our points at about $\pm 30$--$39$ Elo. We obtain 95\% intervals from 1,000 bootstrap resamples of the prompts. Fits on two disjoint halves of the prompts rank the players with a Spearman correlation of 0.989.

\subsection{Measurements}
\label{app:eval-results}

\Cref{tab:eval-main} lists the Elo intervals, speeds and tokens per call of our models, the outside autoregressive models and the human reference in the main pool. \Cref{tab:eval-diffusion} lists the diffusion and flow baselines at every NFE, and \cref{fig:performance-vs-speed} plots the models of both tables.

\begin{table}[t]
    \centering
    \caption{\textbf{Main judge pool: our models, outside autoregressive models and the human reference.} Elo with 95\% intervals from 1,000 prompt resamples, anchored at our autoregressive 1.3B model. ``Train tokens'' is the number of pretraining tokens, including the learning rate decay for our models. Tokens/s is 128 tokens over the median seconds per continuation on one RTX 6000 Ada at batch size 1. ``$\times$ AR'' is the speed relative to our autoregressive model of the same size. $^\ast$We did not train an autoregressive 3B model, so the 3B PLM decoded by AR sampling stands in for it.}
    \label{tab:eval-main}
    \small
    \setlength{\tabcolsep}{4pt}
    \begin{tabular}{lrlrrrr}
        \toprule
        Model & Train tokens & Decoding & Elo [95\% CI] & Tokens/s & Tokens/call & $\times$ AR \\
        \midrule
        AR 350M & 25B & AR sampling & 840 [803, 874] & 746 & 1.00 & 1.00 \\
        PLM-$\beta$=1 350M & 25B & exact & 860 [826, 890] & 1323 & 1.95 & 1.77 \\
        PLM-$\beta$=1 350M & 25B & relaxed & 764 [724, 802] & 2001 & 2.96 & 2.68 \\
        \midrule
        AR 1.3B & 100B & AR sampling & 1000 (anchor) & 310 & 1.00 & 1.00 \\
        PLM-$\beta$=1 1.3B & 100B & exact & 1055 [1025, 1086] & 579 & 2.14 & 1.87 \\
        PLM-$\beta$=1 1.3B & 100B & relaxed & 920 [887, 956] & 856 & 3.21 & 2.76 \\
        \midrule
        PLM-$\beta$=1 3B & 300B & AR sampling$^\ast$ & 1077 [1042, 1111] & 140 & 1.00 & 1.00 \\
        PLM-$\beta$=1 3B & 300B & exact & 1118 [1087, 1153] & 285 & 2.26 & 2.03 \\
        PLM-$\beta$=1 3B & 300B & relaxed & 1064 [1031, 1098] & 430 & 3.49 & 3.06 \\
        \midrule
        Qwen3 4B (base) & 36T & AR sampling & 1191 [1157, 1223] & 102 & 1.00 & \\
        SmolLM3 3B & 11.2T & AR sampling & 1172 [1137, 1207] & 130 & 1.00 & \\
        Llama 3.2 3B & 9T & AR sampling & 1158 [1122, 1192] & 129 & 1.00 & \\
        Qwen2.5 3B & 18T & AR sampling & 1149 [1112, 1185] & 129 & 1.00 & \\
        Gemma 3 4B (pt) & 4T & AR sampling & 1123 [1087, 1159] & 106 & 1.00 & \\
        Qwen3 1.7B (base) & 36T & AR sampling & 1111 [1078, 1144] & 218 & 1.00 & \\
        Gemma 3 1B (pt) & 2T & AR sampling & 1054 [1021, 1091] & 327 & 1.00 & \\
        Llama 3.2 1B & 9T & AR sampling & 1050 [1015, 1085] & 316 & 1.00 & \\
        OLMo 2 1B & 4T & AR sampling & 1011 [975, 1042] & 312 & 1.00 & \\
        Qwen3 0.6B (base) & 36T & AR sampling & 959 [924, 992] & 474 & 1.00 & \\
        SmolLM2 360M & 4T & AR sampling & 952 [916, 989] & 589 & 1.00 & \\
        Pythia 2.8B & 300B & AR sampling & 939 [901, 979] & 152 & 1.00 & \\
        Pythia 1.4B & 300B & AR sampling & 877 [840, 914] & 290 & 1.00 & \\
        SmolLM2 135M & 2T & AR sampling & 825 [787, 859] & 716 & 1.00 & \\
        \midrule
        Human reference & -- & -- & 1264 [1214, 1322] & -- & -- & -- \\
        \bottomrule
    \end{tabular}
\end{table}

\begin{table}[t]
    \centering
    \caption{\textbf{Diffusion and flow baselines in the main judge pool} across the number of denoising steps (NFE): Elo (top of each pair of rows) and tokens/s (bottom). LLaDA and Dream reuse the prompt's key--value cache. All models are timed on one RTX 6000 Ada. Elo intervals are about $\pm 45$--$90$ for these rows, which played fewer matches than the rows of \cref{tab:eval-main}.}
    \label{tab:eval-diffusion}
    \small
    \setlength{\tabcolsep}{4pt}
    \begin{tabular}{llrrrrrrrr}
        \toprule
        & & \multicolumn{8}{c}{NFE} \\
        \cmidrule(l){3-10}
        Model & & 1 & 2 & 4 & 8 & 16 & 32 & 64 & 128 \\
        \midrule
        LLaDA-8B & Elo & & & & 58 & 144 & 315 & 668 & 866 \\
                 & tokens/s & & & & 511 & 250 & 140 & 73 & 36 \\
        \midrule
        Dream-7B & Elo & & & & 110 & 332 & 408 & 676 & 883 \\
                 & tokens/s & & & & 456 & 244 & 128 & 65 & 32 \\
        \midrule
        DFM 1.3B & Elo & & 409 & 476 & 555 & 621 & 665 & 671 & 665 \\
                 & tokens/s & & 10276 & 5068 & 2518 & 1247 & 614 & 305 & 139 \\
        \midrule
        FS-DFM 1.3B & Elo & 357 & 403 & 512 & 599 & 641 & 618 & 626 & 613 \\
                    & tokens/s & 20191 & 9344 & 4789 & 2318 & 1152 & 583 & 286 & 148 \\
        \midrule
        MDLM 350M & Elo & & & & 460 & 523 & 565 & 578 & 578 \\
                  & tokens/s & & & & 3573 & 2136 & 1168 & 606 & 308 \\
        \bottomrule
    \end{tabular}
\end{table}

\section{Post-Training Details}
\label{app:post-training-details}

\subsection{Supervised Fine-Tuning}
\label{app:sft-details}

\paragraph{Data.}
We use four single-turn subsets of the T\"ulu~3 SFT mixture
\cite{lambert2025tulu}: 73,618 general instruction-following examples,
334,252 mathematics examples, 142,275 coding examples, and 29,980 precise
instruction-following examples. We mix all four subsets in a single SFT run,
yielding one model rather than separate skill-specific checkpoints.

We retain the GPT-2 byte-level BPE tokenizer used during pretraining. Each
example is rendered as
\texttt{User: \{prompt\}\textbackslash n\textbackslash nAssistant:
\{response\}<|endoftext|>}, with the response tokenized together with its
leading space. Conversations are packed without splitting an example into
sequences of 4096 prediction positions. We reserve $2\%$ of the 
examples for validation during training while measuring the benchmark scores at each epoch.

\paragraph{Objective.}
Each example is a prompt--response pair $(\vp, \vr)$, which we concatenate
into a sequence $\vx = \vx_{1:L}$ with $\vx_{1:P} = \vp$ and
$\vx_{P+1:L} = \vr$. The response ends with the end-of-document token $x_L$.
SFT applies the joint PLM objective of \cref{eq:plm-final-objective} with
response-only supervision:
\begin{equation}
    \Ll_\mathrm{SFT}(\theta; \vp, \vr) = \underbrace{\sum_{z=1}^{Z} \sum_{\substack{0 \leq k \leq K-1 \\ P < i_z + k < L}} -\log p_{\theta}\!\left(x_{i_z + k} \mid \vx_{< i_z},\, \vu^{(z)}_{i_z : i_z + k}\right)}_{\Ll_\mathrm{dec}^\mathrm{SFT}(\theta; \vx)} + \underbrace{\sum_{i=P+1}^{L} -\log q_\penc(x_i \mid \vx_{<i})}_{\Ll_\mathrm{enc}^\mathrm{SFT}(\penc; \vx)}.
    \label{eq:sft-objective}
\end{equation}
Prompt tokens serve only as context. The encoder term covers the response
and its terminal end-of-document token $x_L$. Auxiliaries are sampled online
from the current autoregressive encoder, independently for the decoder
blocks in which a target appears. The block-grid phase is resampled
between training steps, as in pretraining. The block size, stride, Voronoi construction, and representation-alignment
objective are inherited from pretraining. 

\paragraph{Optimization.}
We fine-tune with global batch size 64 in bf16 mixed precision, using Adam with $(\beta_1,\beta_2)=(0.9,0.95)$, $\epsilon=10^{-8}$, no weight decay, and gradient clipping at norm 1. Each run starts from the final annealed pretraining checkpoint (\cref{app:pretraining-details}) with a fresh optimizer. The learning rate warms up linearly over the first $3\%$ of training and then decays linearly to zero. The packed SFT data contain 137,695 sequences, of which 134,942 are used for training. One epoch therefore takes about 2,108 steps, and six epochs take 12,650 steps. We train for six epochs unless noted otherwise. \Cref{tab:post-train-sweeps} lists the SFT learning rates we ran for each model and the SFT checkpoint (learning rate and epoch) from which DPO starts.

\begin{table}[t]
    \centering
    \caption{\textbf{Post-training sweeps.} We first run SFT, then DPO from the selected SFT checkpoint. The table lists the SFT learning rates (six epochs unless noted), the SFT checkpoint each DPO run starts from, and the DPO learning rate and epoch with the highest GSM8K accuracy. SFT learning rates are in units of $10^{-4}$.}
    \label{tab:post-train-sweeps}
    \small
    \begin{tabular}{llllll}
        \toprule
        & & \multicolumn{2}{c}{SFT} & \multicolumn{2}{c}{DPO} \\
        \cmidrule(lr){3-4} \cmidrule(lr){5-6}
        Model & Stride & Learning rates & Selected (lr, epoch) & Learning rate & Epoch \\
        \midrule
        AR 1.3B  & -- & $\{0.4, 0.8, 1.6, 3.2\}$ & $1.6$, 6 & $10^{-6}$ & 5 \\
        PLM 1.3B & 8  & $\{0.4, 0.8, 1.6, 3.2\}$ & $3.2$, 5 & $10^{-6}$ & 3 \\
        PLM 1.3B & 16 & $1.6$ & $1.6$, 6 & $10^{-6}$ & 5 \\
        PLM 1.3B & 32 & $\{0.4, 0.8, 1.6, 3.2\}$ & $1.6$, 5 & $10^{-6}$ & 2 \\
        PLM 3B   & 16 & $\{0.4, 0.8, 1.6, 3.2\}$ & $1.6$, 6 & $5 \cdot 10^{-7}$ & 6 \\
                 &    & $\{4.8, 6.4, 9.6\}$ (7 epochs) & & & \\
        \bottomrule
    \end{tabular}
\end{table}

\subsection{Preference Optimization}
\label{app:dpo-details}

\paragraph{Preference data.}
We use the T\"ulu~3 preference mixture and retain single-turn examples whose
prompts fall within the same general, mathematics, coding, and precise
instruction-following scope as the SFT data. Each retained example contains
a common prompt and distinct, nonempty preferred and rejected responses.
Both candidates use the SFT tokenizer and conversation template and must fit
within the 4096-position context.

After filtering, the dataset contains 217,916 preference pairs. We use
213,661 for training and 4,255 for validation. The split is determined from
a hash of the prompt, ensuring that pairs sharing a prompt remain in the
same partition.

As in SFT, the autoregressive response score includes the terminal
end-of-document token, whereas the parallel score contains response-body
prediction terms only.

\paragraph{Frozen-reference auxiliaries.}
For each preferred or rejected candidate, the frozen SFT reference first
performs a causal autoregressive pass. Its next-token distribution determines the target cell masses, while the fixed output and auxiliary embeddings define the geometry-aware Voronoi assignment. We use the
same fixed vocabulary of $\Naux=512$ auxiliary vectors and top-$M$ candidate set
with $M=128$ as in pretraining and SFT. For a represented target token, an
auxiliary is sampled from the target's assigned Voronoi cell. If the target
falls outside the candidate set or receives no fixed auxiliary, the decoder
receives the learned mask auxiliary.

Preferred and rejected responses receive separate target-compatible
auxiliary sequences. Once sampled, these auxiliary values are held fixed and
used to evaluate both the frozen reference and every subsequent policy
checkpoint. Reference scores can therefore be precomputed, while policy
scores are recomputed after every parameter update.

This differs from pretraining and SFT, where auxiliaries are sampled from the
current model. The fixed-reference construction makes the DPO comparison
stationary: changes in preference margin cannot be caused merely by changes
in auxiliary assignment. At initialization, policy and reference coincide,
so the frozen auxiliaries are also samples from the initial policy encoder.
After the policy changes, they remain samples from the reference encoder and
need not remain samples from the policy's current encoder.

\paragraph{Deterministic block layout.}

The parallel response score also depends on the placement of decoder
blocks. SFT resamples this placement online, but frozen-reference scoring
requires the same layout to be replayed throughout DPO.

Let $K$ denote the block size, $s$ the stride, and let $a$ denote the
one-based position of the first response target; thus, $a=P+1$ in
\cref{eq:sft-objective}. Let $H(\vp,\vr^+,\vr^-)$ be a stable integer hash
of the preference pair. We assign the pair the deterministic phase
\begin{equation}
o(\vp,\vr^+,\vr^-)
=
\begin{cases}
H(\vp,\vr^+,\vr^-)\bmod s,
& s\leq K,\\[1mm]
\bigl(H(\vp,\vr^+,\vr^-)\bmod K-(a-1)\bigr)\bmod s,
& s>K.
\end{cases}
\label{eq:dpo-deterministic-offset}
\end{equation}
For $s>K$, this guarantees
\begin{equation}
0\leq\bigl((a-1)+o\bigr)\bmod s<K,
\end{equation}
so that the first response target lies inside a decoder block.

Preferred and rejected candidates use the same phase. Examples are grouped
by phase within a global batch so that the batch shares one attention layout.

\paragraph{Stride-dependent decoder score.}
Let $L^{\vr}$ denote the response length and recall that the
encoder-to-decoder ratio is $\beta=s/K$. Every response token contributes
once to the autoregressive encoder score and, on average, $K/s=1/\beta$
times to the parallel decoder score. The combined response therefore
contains approximately
\begin{equation}
L^{\vr}+\frac{L^{\vr}}{\beta}
=
L^{\vr}(1+1/\beta)
\end{equation}
prediction terms, which gives the normalization in
\cref{eq:dpo-plm-score}.

For $\beta<1$ ($s<K$), decoder blocks overlap, and a response token may
be predicted several times from independently sampled auxiliary contexts.
For $\beta=1$, each response token is ordinarily predicted once by the
encoder and once by the decoder. For $\beta>1$ ($s>K$), gaps occur between
decoder blocks, so only a fraction of response positions contributes to
the decoder score, while every response position remains part of the
encoder score. Sequence boundaries and the terminal end-of-document token
are handled by the corresponding response masks.

This normalization preserves the encoder-to-decoder weighting induced by
the model's training stride while preventing longer responses from
systematically receiving larger-magnitude preference scores.

\paragraph{Autoregressive baseline.}
For the autoregressive baseline, we omit the decoder component and use
\begin{equation}
s_\penc^{\mathrm{AR}}(\vp,\vr)
=
-
\frac{
\Ll_{\mathrm{enc}}^{\vr}(\penc;\vx)
}{
L^{\vr}
}.
\label{eq:dpo-ar-score}
\end{equation}
We apply the same reference-relative objective
\cref{eq:plm-dpo-loss} to this score. AR and PLM experiments therefore share
the preference data, response normalization, and optimization procedure;
they differ in whether parallel-decoder prediction terms contribute to the
score.

Because both \cref{eq:dpo-ar-score} and \cref{eq:dpo-plm-score} normalize by
the number of response prediction terms, they define length-normalized DPO
rather than canonical DPO based on summed sequence log-likelihood.

\paragraph{Optimization.}
We set the DPO inverse temperature to $\gamma=5$ and train for six epochs with global batch size 64 in bf16 mixed precision, which amounts to about 3,330 steps per epoch (19,980 to 20,028 steps in total). We use Adam with $(\beta_1,\beta_2)=(0.9,0.95)$, $\epsilon=10^{-8}$, no weight decay, and gradient clipping at norm 1. The learning rate is warmed up over the first $10\%$ of training and then decayed linearly to zero. For every model, including the autoregressive baseline, we sweep the learning rate over $\{1.6 \cdot 10^{-5}, 10^{-6}, 5 \cdot 10^{-7}\}$. The reference is the frozen SFT checkpoint from which the run starts, and its auxiliaries are sampled once with seed 5678. \Cref{tab:post-train-sweeps} lists the learning rate and epoch selected for each model.

\subsection{Post-Training Evaluation}
\label{app:post-training-eval-details}

The benchmark suite comprises GSM8K
\cite{cobbe2021trainingverifierssolvemath}, MATH-500
\cite{hendrycks2021measuringmathematicalproblemsolving,lightman2023letsverifystepstep},
 HumanEval
\cite{chen2021evaluatinglargelanguagemodels}, and IFEval
\cite{zhou2023instructionfollowingevaluationlargelanguage}, covering
mathematical reasoning, code generation, and instruction
following. Parallel decoding quality and throughput are evaluated separately
using the inference methods of \cref{sec:inference}.

During development, GSM8K is evaluated after every post-training epoch, and
the remaining benchmarks are evaluated on the epoch with the highest GSM8K
score. Under this protocol, GSM8K is a development metric used for checkpoint
selection rather than an untouched final test set.

\paragraph{Scope.}
Autoregressive and PLM post-training both start from the full-length pretrained checkpoints of \cref{app:pretraining-details}. We post-trained no autoregressive model at 3B, so the 3B PLM has no autoregressive counterpart.

\paragraph{Benchmarks.}
\Cref{tab:post-train-benchmarks} lists the prompt, generation budget and scoring of each benchmark. We score every question of each benchmark once, with seed 0.

\begin{table}[t]
    \centering
    \caption{\textbf{Post-training benchmarks.} $n$ is the number of questions, all of which we score.}
    \label{tab:post-train-benchmarks}
    \small
    \begin{tabular}{lrp{4.0cm}rp{3.3cm}}
        \toprule
        Benchmark & $n$ & Prompt & Max.\ new tokens & Scoring \\
        \midrule
        GSM8K & 1,319 & zero-shot instruction to reason step by step and end with ``The answer is \textless number\textgreater.'' & 512 & last number in the response \\
        MATH-500 & 500 & ``Please reason step by step, and put your final answer within \texttt{\textbackslash boxed\{\}}.'' & 2,048 & symbolic equivalence of the boxed answer \\
        HumanEval & 164 & implement the function in one Python code block & 768 & pass@1, reference tests executed \\
        IFEval & 541 & the instruction as given & 1,280 & prompt-level strict accuracy \\
        \bottomrule
    \end{tabular}
\end{table}

\paragraph{Decoding.}
We report accuracy at temperature 1. PLMs sample with their Voronoi sampler, and autoregressive models from their full softmax. Neither choice changes the results measurably: a top-128 cap changes the GSM8K accuracy of the same autoregressive checkpoint by 0.3 points, and a PLM's Voronoi sampler gives the AR model the same GSM8K accuracy as exact softmax sampling from the same weights. Greedy decoding scores 2--4 GSM8K points higher than sampling. Over the 1,319 GSM8K questions, the 95\% interval is about $\pm 2.5$. For relaxed inference, we also accept tokens if they are in the $\rho$-nucleus with $\rho=0.94$, which we found to be marginally better than a rule that is based on a fixed threshold $\epsilon$, in this setting.

\paragraph{Checkpoint selection.}
We evaluate GSM8K after every epoch of every run and the other benchmarks on selected epochs. Each row of \cref{tab:post-training} is the checkpoint with the highest GSM8K accuracy in its group, over all learning rates and epochs, its other benchmarks are read from that same checkpoint. All post-training comparisons rest on a single training seed.

\begin{table}[t]
\centering
\caption{\textbf{Detailed Post-training results.} Post-trained 1.3B PLMs are comparable to or better than an autoregressive model at $2.3\times$ to $5.7\times$ the inference speed. The PLM-$\beta=1$ 3B model achieves better quality than all 1.3B variants. ``Q / s'' is questions per second. $^\dagger$We do not train an autoregressive model, so we use the PLM 3B in autoregressive inference mode for comparison.}
\label{tab:post-training}
\resizebox{\textwidth}{!}{\begin{tabular}{lllcccccccc}
\toprule
\multirow{2}{*}[-\dimexpr(\aboverulesep+\belowrulesep+\cmidrulewidth)*1/2\relax]{Size} & \multirow{2}{*}[-\dimexpr(\aboverulesep+\belowrulesep+\cmidrulewidth)*1/2\relax]{$\beta$} & \multirow{2}{*}[-\dimexpr(\aboverulesep+\belowrulesep+\cmidrulewidth)*1/2\relax]{Decoding} & \multicolumn{4}{c}{GSM8K} & \multicolumn{4}{c}{MATH-500} \\
\cmidrule(lr){4-7}\cmidrule(lr){8-11}
 & & & Score & Q / s & Tok / s & Tok / call & Score & Q / s & Tok / s & Tok / call \\
\midrule
\multirow{10}{*}[-\dimexpr(\aboverulesep+\belowrulesep+\cmidrulewidth)*3/2\relax]{1.3B} & AR & AR & 27.9 & 0.9 (1.0\texttimes) & 291 & 1.0 & 7.4 & 0.4 (1.0\texttimes) & 291 & 1.0 \\
\cmidrule{2-11}
 & \multirow{3}{*}{$0.5$} & AR & 27.0 & 1.1 (1.2\texttimes) & 308 & 1.0 & 7.6 & 0.4 (1.0\texttimes) & 308 & 1.0 \\
 &  & Exact & 31.5 & 4.4 (4.8\texttimes) & 1112 & 4.2 & 8.0 & 1.0 (2.6\texttimes) & 803 & 3.2 \\
 &  & Relaxed & \textbf{32.6} & \textbf{5.2 (5.7\texttimes)} & \textbf{1337} & \textbf{6.1} & \textbf{8.4} & \textbf{1.4 (3.7\texttimes)} & \textbf{1107} & \textbf{5.1} \\
\cmidrule{2-11}
 & \multirow{3}{*}{$1$} & AR & 27.7 & 1.1 (1.2\texttimes) & 307 & 1.0 & 5.6 & 0.4 (1.1\texttimes) & 309 & 1.0 \\
 &  & Exact & 28.1 & 4.0 (4.4\texttimes) & 1090 & 4.1 & 6.2 & 1.0 (2.7\texttimes) & 773 & 3.0 \\
 &  & Relaxed & 29.4 & 4.9 (5.4\texttimes) & 1304 & 5.7 & 7.2 & 1.2 (3.1\texttimes) & 941 & 4.3 \\
\cmidrule{2-11}
 & \multirow{3}{*}{$2$} & AR & 27.0 & 1.0 (1.1\texttimes) & 307 & 1.0 & 6.8 & 0.4 (1.0\texttimes) & 308 & 1.0 \\
 &  & Exact & 28.8 & 4.0 (4.4\texttimes) & 1059 & 3.9 & 7.8 & 0.9 (2.4\texttimes) & 757 & 2.9 \\
 &  & Relaxed & 29.5 & 4.8 (5.2\texttimes) & 1251 & 5.5 & \textbf{8.4} & 1.0 (2.8\texttimes) & 882 & 4.4 \\
\midrule
\multirow{3}{*}{3B} & \multirow{3}{*}{$1$} & AR & \textbf{39.0} & 0.4 (1.0\texttimes) & 140 & 1.0 & 12.2 & 0.2 (1.0\texttimes) & 140 & 1.0 \\
 &  & Exact & 37.2 & 1.8 (4.1\texttimes) & 558 & 4.5 & 12.6 & 0.5 (3.0\texttimes) & 406 & 3.3 \\
 &  & Relaxed & 37.4 & \textbf{2.1 (4.8\texttimes)} & \textbf{662} & \textbf{5.4} & \textbf{14.8} & \textbf{0.7 (3.8\texttimes)} & \textbf{516} & \textbf{4.3} \\
\midrule
\midrule
\multirow{2}{*}[-\dimexpr(\aboverulesep+\belowrulesep+\cmidrulewidth)*1/2\relax]{Size} & \multirow{2}{*}[-\dimexpr(\aboverulesep+\belowrulesep+\cmidrulewidth)*1/2\relax]{$\beta$} & \multirow{2}{*}[-\dimexpr(\aboverulesep+\belowrulesep+\cmidrulewidth)*1/2\relax]{Decoding} & \multicolumn{4}{c}{HumanEval} & \multicolumn{4}{c}{IFEval} \\
\cmidrule(lr){4-7}\cmidrule(lr){8-11}
 & & & Score & Q / s & Tok / s & Tok / call & Score & Q / s & Tok / s & Tok / call \\
\midrule
\multirow{10}{*}[-\dimexpr(\aboverulesep+\belowrulesep+\cmidrulewidth)*3/2\relax]{1.3B} & AR & AR & 12.8 & 0.7 (1.0\texttimes) & 289 & 1.0 & 48.2 & 0.8 (1.0\texttimes) & 290 & 1.0 \\
\cmidrule{2-11}
 & \multirow{3}{*}{$0.5$} & AR & \textbf{19.5} & 0.9 (1.3\texttimes) & 308 & 1.0 & 45.7 & 0.9 (1.2\texttimes) & 307 & 1.0 \\
 &  & Exact & 16.5 & 3.3 (4.9\texttimes) & 1113 & 4.3 & 46.2 & 1.6 (2.1\texttimes) & 525 & 1.9 \\
 &  & Relaxed & 17.7 & 3.0 (4.5\texttimes) & 1323 & 5.9 & 45.1 & \textbf{1.9 (2.4\texttimes)} & \textbf{780} & \textbf{3.5} \\
\cmidrule{2-11}
 & \multirow{3}{*}{$1$} & AR & 8.5 & 1.0 (1.5\texttimes) & 306 & 1.0 & 45.3 & 0.9 (1.2\texttimes) & 308 & 1.0 \\
 &  & Exact & 10.4 & 4.6 (6.8\texttimes) & 1131 & 4.3 & 44.7 & 1.4 (1.9\texttimes) & 485 & 1.8 \\
 &  & Relaxed & 10.4 & \textbf{5.2 (7.8\texttimes)} & \textbf{1348} & \textbf{6.2} & 45.7 & \textbf{1.9 (2.5\texttimes)} & 713 & 3.1 \\
\cmidrule{2-11}
 & \multirow{3}{*}{$2$} & AR & 11.0 & 0.8 (1.2\texttimes) & 307 & 1.0 & 48.2 & 0.9 (1.2\texttimes) & 308 & 1.0 \\
 &  & Exact & 17.7 & 2.8 (4.2\texttimes) & 956 & 3.6 & \textbf{51.0} & 1.4 (1.8\texttimes) & 478 & 1.8 \\
 &  & Relaxed & 15.8 & 3.0 (4.5\texttimes) & 1084 & 5.0 & 48.1 & 1.8 (2.3\texttimes) & 718 & 3.1 \\
\midrule
\multirow{3}{*}{3B} & \multirow{3}{*}{$1$} & AR & 23.2 & 0.4 (1.0\texttimes) & 140 & 1.0 & \textbf{54.0} & 0.4 (1.0\texttimes) & 140 & 1.0 \\
 &  & Exact & 26.8 & 1.4 (3.8\texttimes) & 558 & 4.6 & 52.7 & 0.7 (1.7\texttimes) & 235 & 1.9 \\
 &  & Relaxed & \textbf{32.3} & \textbf{1.5 (4.1\texttimes)} & \textbf{644} & \textbf{5.3} & 53.6 & \textbf{1.0 (2.5\texttimes)} & \textbf{360} & \textbf{2.9} \\
\bottomrule
\end{tabular}
}
\end{table}

\paragraph{Speed.}
To maximize the tokens per second we can decode, we investigate the effects of different proposal lengths $\Kpred \in \{1, 2, 4, 7, 8, 10, 16\}$ tokens under exact verification. Then, the sampling distribution is equal and we only measure the speed changes. \Cref{tab:post-train-spec} reports the speedups over $\Kpred$ on GSM8K, on one NVIDIA RTX 6000 Ada GPU. This informs our choice of $\Kpred=7$.

\section{Scaling Law Fits}
\label{app:scaling-law-fits}

\paragraph{Formal definition of reconstructed tokens.} \label{app:formal-reconstructed-tokens}
Throughout this definition, all encoder and decoder distributions are additionally conditioned on the clean text preceding the block, which we omit from the notation.
Given a block $\vx$ of $K$ tokens, randomly sample $\vu \sim q_\penc(\vu \mid \vx)$ and then reconstruct the tokens via the parallel decoder $p_\theta(\vx \mid \vu)$ as follows:
\begin{equation}
    \hat x_i = \begin{cases}
        x_i, & i = 1, \\
        \argmax_{v \in \{1, \dots, V\}} p_\theta(x_i = v \mid \vu_{\leq i}), & 2 \leq i \leq K.
    \end{cases}
\end{equation}
Here, the first token is counted as correct because at inference we can retrieve it from the autoregressive branch of our model $q_\penc(x_i \mid \vx_{<i})$.
We then define the average number of reconstructed tokens as:
\begin{equation}
    R = \EE\!\left[\max \{ c \in \{1, \dots, K\} \mid \hat x_j = x_j \,\, \forall 1 \leq j \leq c \}\right],
\end{equation}
that is, we count how many tokens from the start of the completion the decoder has accurately reconstructed, averaged over training sequences and sampled auxiliaries. An autoregressive model has $R = 1$.

\paragraph{Analytic form.}
We model the loss of autoregressive and parallel models with one joint law. The autoregressive part is the separable law of \citet{hoffmann2022training} in the non-embedding parameter count $N$ and the training tokens $D$, an irreducible loss $E$ plus a finite-model-size and a finite-data term,
\begin{equation}
    \Ll_\mathrm{AR}(N, D) = E + A\, n^{-a} + B\, d^{-b}, \qquad n = \frac{N}{10^8}, \quad d = \frac{D}{10^{10}}.
    \label{eq:app-ar-law}
\end{equation}
A PLM's encoder-to-decoder ratio is $\beta = s/K$, so each encoder target corresponds to $1/\beta$ decoder targets on average. We model its encoder loss as that of an autoregressive model with a smaller effective parameter count trained on the same number of tokens,
\begin{equation}
    \Ll_\mathrm{enc}(N, D, \beta) = \Ll_\mathrm{AR}\!\left(\frac{N\beta}{\beta+c},\, D\right)
    = E + \frac{A}{n_\mathrm{eff}^{a}} + \frac{B}{d^{b}},
    \qquad n_\mathrm{eff} = \frac{n\beta}{\beta+c}.
    \label{eq:app-plm-law}
\end{equation}
The effective fraction of parameters available for language modeling is $\beta/(\beta+c)$, with a relative capacity reduction of $c/(\beta+c)$. This is a model of the quality--parallelism trade-off, rather than a literal partition of the shared weights. The autoregressive law is recovered as $\beta\to\infty$. We model $R$ as one token plus a power law in model size, training tokens, and the inverse encoder-to-decoder ratio,
\begin{equation}
    R(N, D, \beta) = 1 + A_R\, \frac{n^{\alpha_N}\, d^{\alpha_D}}{\beta^{\alpha_\beta}}.
    \label{eq:app-correct-law}
\end{equation}
Here $A_R$ is the reconstruction prefactor, and $\alpha_N$, $\alpha_D$, and $\alpha_\beta$ are fitted exponents. The fitted positive exponents imply that $R$ increases with model size and training tokens and decreases with $\beta$, approaching $1$ as $\beta\to\infty$. The power law describes the fitted range; it does not model saturation at the block-size ceiling $K=16$. \Cref{eq:scaling-laws} in the main text is \cref{eq:app-plm-law,eq:app-correct-law} with the fitted coefficients inserted and rounded.

\paragraph{Fitting.}
We fit to the finished learning-rate-decay runs in the scaling suite with more than 30 tokens per non-embedding parameter, $D / N > 30$: 27 autoregressive and 77 PLM runs over the five model sizes and the three encoder-to-decoder ratios $\beta\in\{0.5,1,2\}$. Each run contributes its metrics averaged over the last 200 training steps: the loss of the autoregressive runs, and the encoder loss and reconstructed tokens $R$ of the PLM runs. \Cref{eq:app-ar-law,eq:app-plm-law} are one least-squares problem: all six coefficients $E, A, a, B, b, c$ minimize the sum of squared residuals in the loss over the 104 runs, with autoregressive runs evaluated in the limit $\beta\to\infty$. \Cref{eq:app-correct-law} is fitted by least squares in $R$ to the 77 PLM runs. $A$, $B$ and $A_R$ are fitted as logarithms to keep them positive; every fit uses \texttt{scipy.optimize.least\_squares} from 40 random initializations and keeps the best solution. The 3B size contributes PLM-$\beta$=1 runs only. Model names refer to total parameter counts, while $N$ excludes embeddings (\cref{tab:pretraining-sizes}).

\paragraph{Uncertainty.}
Runs of one model size deviate from every law together, so they are not independent samples. We therefore take the model size as the unit of uncertainty: each law is refitted without one of the five sizes in turn, and the standard deviation of a coefficient $\gamma$ is the delete-one jackknife estimate $\sqrt{\tfrac{g-1}{g} \sum_{m=1}^{g} (\gamma_{(m)} - \bar\gamma)^2}$ for $g = 5$ sizes, where $\gamma_{(m)}$ is the coefficient refitted without size $m$ and $\bar\gamma$ is the mean over the $g$ refits. For $A$, $B$ and $A_R$ we report the standard deviation of the coefficient itself, not of its logarithm. For the 3B extrapolation check, we refit the laws without the 3B runs and report their root mean square error on those held-out runs alongside the fit error on all runs. The coefficients in the main text and \cref{tab:scaling-law-coefficients} use all five sizes; the 3B hold-out errors come from these separate refits.

\begin{table}
    \centering
    \caption{\textbf{Fitted coefficients of the scaling laws.} Values with one jackknife standard deviation from refits leaving out one model size at a time; $n = N / 10^8$ and $d = D / 10^{10}$. RMSE is the root mean square error of the fit on all runs and, for the hold-out, of the refit without the 3B size evaluated on the 3B runs. The AR hold-out uses 1.3B runs instead, as noted below.}
    \label{tab:scaling-law-coefficients}
    \begin{tabular}{llr@{\,$\pm$\,}l}
        \toprule
        Law & Coefficient & \multicolumn{2}{c}{Value} \\
        \midrule
        \multirow{6}{*}{$\Ll_\mathrm{enc}$}
            & $E$      & 1.105  & 0.12  \\
            & $A$      & 0.851  & 0.11  \\
            & $a$      & 0.194  & 0.023 \\
            & $B$      & 0.763  & 0.23  \\
            & $b$      & 0.127  & 0.032 \\
            & $c$      & 0.482  & 0.027 \\
        \midrule
        \multirow{4}{*}{$R$}
            & $A_R$            & 1.051  & 0.017 \\
            & $\alpha_\beta$   & 0.160  & 0.005 \\
            & $\alpha_N$       & 0.120  & 0.010 \\
            & $\alpha_D$       & 0.094  & 0.006 \\
        \bottomrule
    \end{tabular}
    \hspace{2em}
    \begin{tabular}{lrr}
        \toprule
        Law & RMSE fit & RMSE hold-out \\
        \midrule
        $\Ll_\mathrm{AR}$  & 0.0073 & 0.0065$^\dagger$ \\
        $\Ll_\mathrm{enc}$ & 0.0057 & 0.0051 \\
        $R$                & 0.024  & 0.010  \\
        \bottomrule
    \end{tabular}
    \\[0.5em]
    {\small $^\dagger$There is no autoregressive 3B run. For $\Ll_\mathrm{AR}$, the joint loss law is refitted without any 1.3B runs and evaluated on the held-out 1.3B autoregressive runs.}
\end{table}

\section{Inference}
\label{app:inference}

\begin{figure}[t]
    \centering
    \includegraphics[width=\linewidth]{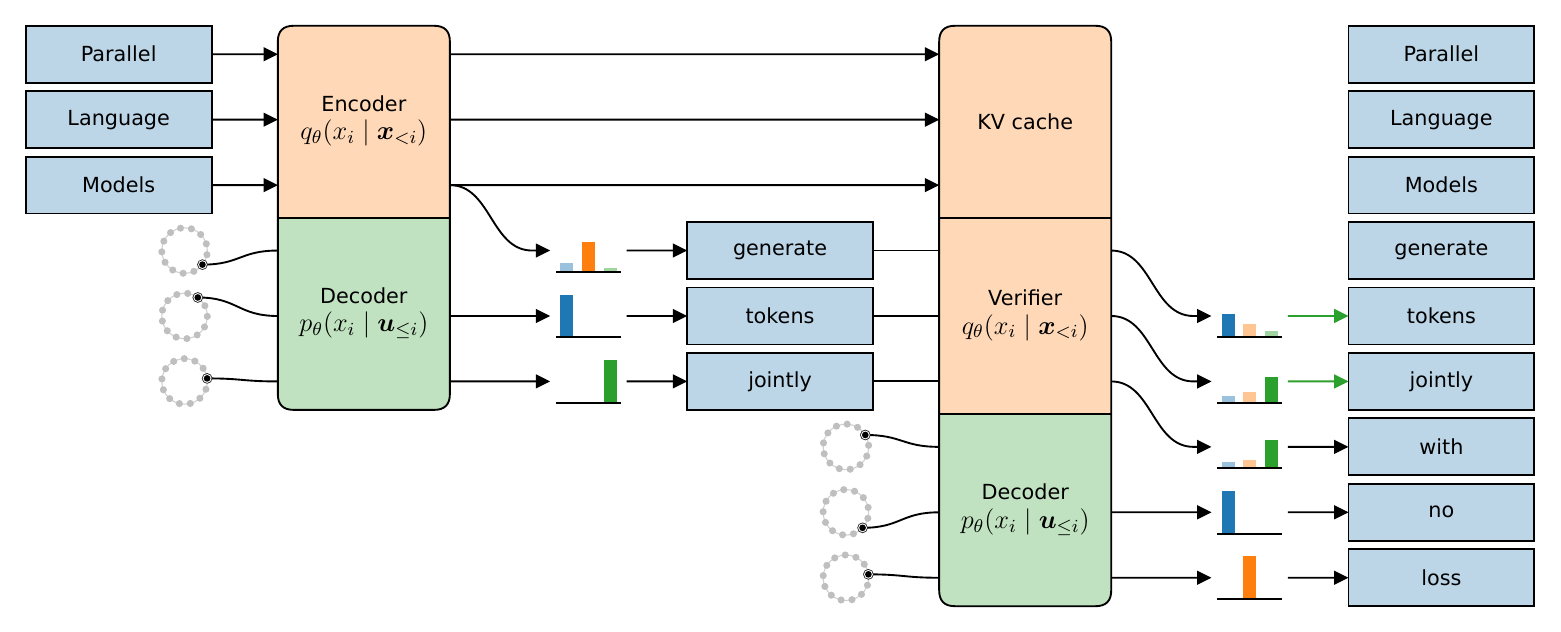}
    \caption{\textbf{Parallel Language Models predict coordinated tokens in parallel.} This schematic shows a simplified version of our inference algorithm. Our encoder and decoder share a transformer backend and can be called in parallel via a custom attention mask. First, based on randomly drawn Voronoi auxiliaries, several tokens are proposed in parallel. The first token is always accepted as it can be chosen autoregressively. In the following call the tokens are verified using the encoder while new tokens are proposed simultaneously. The KV cache is built from verified correct tokens.}
    \label{fig:inference-scheme}
\end{figure}

\begin{algorithm}[t]
\caption{Approximate Sampling with Quadratic Decoding}
\label{alg:error-correction-quadratic}
\begin{algorithmic}
    \Require Prompt $\vx_{<i}$, PLM $p_\theta$, width $\Kpred$, nucleus $\rho$ or threshold $\epsilon$, target length $L$
    \State Sample $u_j \sim p(u)$ for all $i \leq j < L + 2\Kpred$ \Comment{cf.~\cref{sec:auxiliaries}}
    \State $x_j \gets x_\mathrm{pad}$ for all $i \leq j < i + \Kpred - 1$
    \Comment{forces $m=0$ in the first round}
    \While{$i < L$}
    \State $p_{mj} \gets p_\theta(x_{j+m} \mid \vx_{<i+m}, \vu_{i+m:j+m})$,
    \State $\hspace{3em}$ jointly for all $0 \leq m < \Kpred$, $i \leq j < i + \Kpred$
    \State $\hat{x}_{mj} \gets \argmax_v \; p_{mjv}$
    \State $q_j \gets q_\penc(x_j \mid \vx_{<j})$, jointly for all $i \leq j < i + \Kpred$ \Comment{$q_\penc$ and $p_\theta$ can be run in parallel}
    \State $\tilde x_j \gets f^{\mathrm{Voronoi}}(u_j; q_j)$ for all $i \leq j < i + \Kpred$
    \Comment{cf. \Cref{sec:auxiliaries}}
    \State $a_j \gets \big[ x_j = \tilde x_j \big] \lor \big[ \textstyle\sum_{v \,:\, q_{jv} > q_{j x_j}} q_{jv} < \rho \big] \lor \big[ q_{jx_j} > \epsilon \big]$
    \Comment{$\rho = 0, \epsilon = 1$ is exact}
    \State $\hspace{2.5em}$ for all $i \leq j < i + \Kpred - 1$
    \State $m \gets \min \big( \{ j : i \leq j < i + \Kpred - 1,\; \lnot a_j \} \cup \{ i + \Kpred - 1 \} \big) - i$
    \State $x_{i+m} \gets \tilde x_{i+m}$
    \State $x_{j+m} \gets \hat{x}_{mj}$ for $i < j < i + \Kpred$
    \State $i \gets i + m + 1$
    \EndWhile
    \State \Return $\vx_{\leq L}$
\end{algorithmic}
\end{algorithm}

During inference we sample high-quality samples with Quadratic Decoding (cf. \Cref{sec:inference}). Pseudo-code can be found in \Cref{alg:error-correction-quadratic}. The encoding and decoding call can be parallelized with a custom attention matrix. The case where we accept all proposed tokens is illustrated in \Cref{fig:inference-scheme}.

\subsection{Measurement Protocol}
\label{app:inference-measurement}

Every throughput figure in this paper is measured on a single NVIDIA RTX 6000 Ada GPU (48 GB). Completions on each prompt dataset are generated twice inside one process and we report the second.
The first pass absorbs engine initialization, CUDA graph capture, kernel autotuning and warms up the GPU for accurate measurements under full load. Compared against single-pass runs, throughput is up to $4.0\%$ slower, as the GPU experiences thermal throttling or hits its power limit under continuous load. This more conservative number should be closer to real-world loads. Unless otherwise stated, we report end-to-end tokens per second: total tokens emitted divided by total wall clock, prefill of the kv cache included. For the pretraining comparison, speed is 128 tokens divided by the median seconds per continuation (see \cref{app:eval-speed}).

\subsection{Autoregressive}
\label{app:ar}

We leverage the standard vLLM inference for autoregressive models.

\subsection{Parallel Language Models}
\label{app:ptp}

\paragraph{Round geometry.} One round decodes $W = \Kpred + \Kpred^2$ query positions: $\Kpred$ verify
rows and $\Kpred$ proposal blocks of $\Kpred$ rows each. The first verify row is the token
committed by the previous round, which has no cached key or value yet; the remaining
$\Kpred-1$ carry the pending hypothesis. This maps exactly onto vLLM's fixed-$\Kpred$ speculative
path with \texttt{num\_speculative\_tokens} $= W - 1$, so one engine step is one round
with no padding or bucketing.

\paragraph{Precision.} We decode in float16 rather than bfloat16, which is $3.0\%$ faster at comparable quality.

\paragraph{Round width.} Fitted per acceptance rule based on best tokens per second a sweep over a subset of prompts from the training set, $\Kpred = 6$ for all pretrained models, $\Kpred = 10$ for post-trained models up to 1.3B parameters under the relaxed acceptance rule, $\Kpred = 7$ otherwise. See \Cref{tab:post-train-spec} for an example.

\paragraph{Acceptance rules.} Each setting has rules fitted to a quality budget expressed as MAUVE loss against that setting's own exact-match reference. We choose to either accept tokens in the $\rho$-nucleus of the encoder likelihood or tokens that have an encoder likelihood higher than $\epsilon$, with $\rho$ and $\epsilon$ fitted such that we lose no more than $1\%$ in MAUVE score. This yields $\epsilon=2 \cdot 10^{-3}$ for pretrained models and $\rho=0.94$ for post-trained models. 

\subsection{Discrete Flow Matching}
\label{app:flow}

We time DFM 1.3B and FS-DFM \cite{monsefi2026fsdfm} with the following changes to the authors' code: We cast the block weights to bfloat16, compile the model with \texttt{torch.compile}, and cache the sinusoidal timestep frequencies, so they do not have to be recomputed every forward call.

\begin{table}[t]
    \centering
    \caption{\textbf{Quadratic Decoding of PLMs on GSM8K is optimal for $\Kpred=7$ under exact acceptance}. Higher widths $\Kpred$ accept more tokens per call but are slower per call, with GPU-specific jumps after $\Kpred=7, 10$. Across $\Kpred$, accuracy varies only within sampling noise.}
    \label{tab:post-train-spec}
    \small
    \begin{tabular}{cccccccc}
        \toprule
        \multicolumn{2}{c}{Model} & \multicolumn{3}{c}{PLM 1.3B, SFT + DPO, $\beta=1$} & \multicolumn{3}{c}{PLM 3B, SFT + DPO, $\beta=1$} \\
        \cmidrule(lr){3-5} \cmidrule(lr){6-8}
        & & tokens/s & tokens/call & ms/call & tokens/s & tokens/call & ms/call \\
        \midrule
        \multicolumn{2}{c}{AR} & 307 & 1.00 & 3.25 & 140 & 1.00 & 7.14 \\
        \midrule
        \multirow{6}{*}{$\Kpred$} & 2  & 548  & 1.85 & 3.32 & 255 & 1.88 & 7.28 \\
                             & 4  & 876  & 3.06 & 3.40 & 425 & 3.25 & 7.46 \\
                             & 7  & \textbf{1090} & 4.05 & 3.59 & \textbf{558} & 4.53 & 7.86 \\
                             & 8  & 966  & 4.25 & 4.25 & 514 & 4.81 & 9.07 \\
                             & 10 & 963  & 4.52 & 4.53 & 533 & 5.22 & 9.47 \\
                             & 16 & 569  & 4.85 & 8.28 & 290 & 5.83 & 19.56 \\
        \midrule
        \multicolumn{2}{c}{GSM8K} & \multicolumn{3}{c}{27.9--28.1} & \multicolumn{3}{c}{37.2--37.8} \\
        \bottomrule
    \end{tabular}
\end{table}

\subsection{Masked Diffusion}
\label{app:diffusion}

We include three models, all decoding a fixed-length canvas by iterative denoising, where $r_t$ is the noise level at timestep $t$, e.g., related to the fraction of masked tokens.

\paragraph{LLaDA-8B \cite{nie2025large}} Temperature 0, no top-$k$, block length 32. Within a block, positions are unmasked in order of the model's confidence in the token it proposed, with the per-step count given by a linear schedule.

\paragraph{Dream-7B \cite{ye2025dream}} Negative-entropy remasking, temperature 0.2, nucleus 0.95, no top-$k$. The schedule is continuous-time: the fraction transferred at step $t$ is $1 - r_{t+1}/r_t$, where $r_t$ is linearly spaced between $1$ and $\varepsilon = 10^{-3}$, and the number of positions moved is that fraction of those still masked. Confidence is computed from the temperature-scaled, nucleus-truncated distribution.

\paragraph{MDLM-350M \cite{sahoo2024simple}} Our own 350M training of the MDLM objective, not a released checkpoint. We serve the exponential moving average of the weights. No temperature and no truncation. The sampler is ancestral: mask and token are drawn together from one float64 categorical per masked position, the mask holding $r_{t+1}/r_{t}$, where $r_t$ is linearly spaced between $1$ and $0$. There is no block structure, the whole canvas is denoised at once, and the last step unmasks whatever remains.

\paragraph{KV cache.} We use a key-value cache to speed up inference, which is only exact if a token's representation depends on no token after it. As the diffusion models above attend bidirectionally across the whole sequence, including the prompt, reusing the prompt's keys and values is an approximation.

\section{Sample Generations}
\label{app:sample_generations}
We provide example sentence completions from our models and a subset of baselines on two random documents from FineWeb-Edu \cite{lozhkov2024fineweb-edu}.

\IfFileExists{completions/prompt3/prompt.tex}{
    \begin{promptbox}
    \textbf{Prompt}\\[2pt]
    \small
    Unlocking the Wonders of STEM Education \\
Welcome to our senior secondary school, where we are redefining education through a Science, Technology, Engineering, and Mathematics (STEM) curriculum. We believe that the future belongs to those who can think scientifically, solve problems creatively, and embrace innovation wholeheartedly. Our aim is to ignite

    \end{promptbox}
}{}

\IfFileExists{completions/prompt3/human_reference.tex}{
    \begin{generationbox}{Human Reference}
    \small
    the sparks of science in every student, equipping them with the essential skills and knowledge to thrive in the 21st century. \\
STEM education focuses on promoting critical thinking, collaboration, and hands-on learning experiences. Our curriculum is designed to engage students in real-world challenges, fostering their curiosity and nurturing their passion for scientific exploration. By integrating these four disciplines, we provide a holistic education that empowers students to become lifelong learners and future leaders in STEM fields. \\
Embracing a Dynamic Learning Environment \\
At our senior secondary school, we believe that learning should be dynamic and interactive. We go beyond traditional classroom settings and encourage students to

    \end{generationbox}
}{}

\IfFileExists{completions/prompt3/vllm_ptp_3b_s16_300b_decay114436_pre_exact.tex}{
    \begin{generationbox}{PLM 3B Exact}
    \small
    your passion for STEM, fostering a deep understanding and appreciation of the subject matter. It has been said that this is the best STEM school. If you want to grow, you can book admissions. \\
World-Class Infrastructure \\
Behind every successful STEM institution lies well-equipped laboratories, cutting-edge facilities, and modern classrooms. At Saadiyya School, we take these aspects very seriously. Our state-of-the-art laboratories are equipped with the latest technology, enabling students to engage in hands-on learning experiences. Our classrooms are designed to foster collaboration and critical thinking, encouraging students to explore new ideas and push the boundaries of

    \textbf{Generation Speed: 285 tok/s}
    \end{generationbox}
}{}

\IfFileExists{completions/prompt3/vllm_ptp_3b_s16_300b_decay114436_pre_relaxed_1}{
    \begin{generationbox}{PLM 3B Relaxed}
    \small
    your passion, nurture your skills, and prepare you for a world that's rapidly evolving. \\
Why Choose STEM at Senior Secondary School? \\
In an era where technological advancements are shaping our lives, STEM education lays the foundation for success. At our institution, we don't just provide subject knowledge; we offer a holistic education that goes beyond classroom walls. Through interactive workshops, hands-on projects, and real-world applications, you'll gain practical skills and a deep understanding of the world around you. \\
Our Approach to STEM Education \\
Our STEM curriculum goes beyond traditional subjects, equipping you with skills like critical thinking, problem-solving

    \textbf{Generation Speed: 430 tok/s}
    \end{generationbox}
}{}

\IfFileExists{completions/prompt3/vllm_ptp_1_3b_s16_80b_decay76293_pre_relaxed_1.tex}{
    \begin{generationbox}{PLM 1.3B Relaxed}
    \small
    a passion for STEM among our students, preparing them with a skillset that is increasingly sought in the modern world. 

Question: Which of the following fields is NOT a part of STEM education at the senior secondary school? A) Science B) History C) Technology D) Mathematics Answer: B) History

Question: What is the primary goal of the senior secondary school's STEM curriculum? Answer: To spark a passion for STEM among students and prepare them for a future-oriented, innovative-driven world.

Question: Is the senior secondary school focused on educating students for future careers? Answer: Yes. 

Question

    \textbf{Generation Speed: 856 tok/s}
    \end{generationbox}
}{}

\IfFileExists{completions/prompt3/vllm_ptp_350m_s16_20b_decay38146_pre_relaxed_1.tex}{
    \begin{generationbox}{PLM 350M Relaxed}
    \small
    a love of learning and equips young students with the necessary skills to thrive in the ever-evolving world.

At a glance

Question: What are the three core subjects that the curriculum focuses on in senior secondary? Answer: Science, Technology, Engineering, and Mathematics (STEM)

Question: What is the primary focus of the curriculum for senior secondary school? Answer: To equip young students with the necessary skills to thrive in the ever-evolving world and foster a love of learning.

Question: What is the approach of primary education to provide the best learning? Answer: The approach emphasizes active and creative participation

    \textbf{Generation Speed: 2001 tok/s}
    \end{generationbox}
}{}

\IfFileExists{completions/prompt3/vllm_qwen3_4b_base.tex}{
    \begin{generationbox}{Qwen3 4B}
    \small
    a passion for discovery that will empower our students to tackle the world's biggest challenges and build a better tomorrow. \\
Fostering Critical Thinking and Problem-Solving Skills \\
At our school, we understand that STEM education goes beyond rote learning. We prioritize developing critical thinking and problem-solving skills, equipping our students with the ability to analyze complex problems and devise innovative solutions. Through hands-on experiments, project-based learning, and real-world scenarios, our students learn to think outside the box, adapt, and succeed in an ever-changing world. \\
Hands-on Learning: Making Science Come Alive \\
Our STEM-focused curriculum transforms learning into an engaging experience. Our

    \textbf{Generation Speed: 102 tok/s}
    \end{generationbox}
}{}

\IfFileExists{completions/prompt3/vllm_llama_3_2_1b.tex}{
    \begin{generationbox}{Llama 3.2 1B}
    \small
    the curiosity and passion for learning in our students, so that they may thrive in a world of opportunities and challenges. In this blog post, we will delved into the wonders of our STEM education programs with the aim of providing actionable insights and inspiration for parents and students alike. We have handpicked some of the best STEM education programs that we know best in this post. Read on to discover how your child’s academic journey can be enriched by these hands-on, state-of-the-art learning experiences.  \\
The Value of STEAM Education Programs  \\
The foundation of STEAM education programs is centered around fostering students’ skills in engineering, computation, social

    \textbf{Generation Speed: 316 tok/s}
    \end{generationbox}
}{}

\IfFileExists{completions/prompt3/vllm_smollm2_360m.tex}{
    \begin{generationbox}{SmolLM2 360M}
    \small
    curiosity and spark curiosity among our students.

The course structure follows the American School of Mankato's traditional 14 weeks model, allowing for flexibility and adaptation based on the student's interests and needs. Through this integrated approach, students engage in hands-on, project-based learning experiences that foster critical thinking, communication, collaboration, and problem-solving skills.

Our Philosophy of Teaching and Learning \\
We believe that learning should be enjoyable, both for the student and for the teacher. By emphasizing critical thinking, creativity, and passion, we create engaging lessons that spark joy within our young learners. Moreover, our commitment to

    \textbf{Generation Speed: 589 tok/s}
    \end{generationbox}
}{}

\IfFileExists{completions/prompt8/prompt.tex}{
    \begin{promptbox}
    \textbf{Prompt}\\[2pt]
    \small
    What You Should Know About Cortisone Shots \\
What are cortisone shots? \\
Cortisone shots are another term for steroid injections, as cortisone is a type of corticosteroid. There are many different types of steroid injections that can be used for many different reasons. In general, cort

    \end{promptbox}
}{}

\IfFileExists{completions/prompt8/human_reference.tex}{
    \begin{generationbox}{Human Reference}
    \small
    isone shots are used for relieving painful conditions that are related to something going wrong in the musculoskeletal system, whether in a joint, a tendon, muscle or the spine. \\
What conditions are cortisone shots used to treat? \\
Cortisone shots, also called steroid shots, are used to treat conditions in which there is pain related to inflammation, whether it's in a large joint such as the knee or shoulder or smaller joints as in the hand or ankle. \\
They can also be used for painful and inflammatory conditions from a herniated disc in the spine, whether it's in the lumb

    \end{generationbox}
}{}

\IfFileExists{completions/prompt8/vllm_ptp_3b_s16_300b_decay114436_pre_exact.tex}{
    \begin{generationbox}{PLM 3B Exact}
    \small
    icosteroid injections target the source of an underlying inflammatory condition. Steroid injections can target bone, joints, the skin, or the nerves as well as different types of tissues. \\
The cortisone injections that are frequently used for inflammatory conditions can include injections that are for the knees, elbows, spine, shoulders and wrists, and the cortisone shot may be placed in different parts of the anatomy depending on where the pain occurs. \\
Where should I avoid getting them? \\
Cortisone shots need to be reserved for cases where there's definitive evidence that a steroid injection is the most likely solution. First, the location of

    \textbf{Generation Speed: 285 tok/s}
    \end{generationbox}
}{}

\IfFileExists{completions/prompt8/vllm_ptp_3b_s16_300b_decay114436_pre_relaxed_1}{
    \begin{generationbox}{PLM 3B Relaxed}
    \small
    icosteroid injections are used to reduce pain and inflammation. These injections are the preferred form of corticosteroid therapy over oral corticosteroids since they are directly targeted to the area of the body that is affected with pain. Oral corticosteroids are commonly prescribed to treat inflammatory conditions of the body (arthritis, colitis, IBS, etc.) but are primarily intended to act systemically to provide inflammation. The downside is systemic side effects can have many negative side effects, which is why corticosteroid injections are preferred to treat pain or arthritis of a specific joint area. \\
Cortisone injections are very

    \textbf{Generation Speed: 430 tok/s}
    \end{generationbox}
}{}

\IfFileExists{completions/prompt8/vllm_ptp_1_3b_s16_80b_decay76293_pre_relaxed_1.tex}{
    \begin{generationbox}{PLM 1.3B Relaxed}
    \small
    isone can be very helpful. However, before you treat with it, it's important to be an informed participant and choose the right route. \\
Cortisone injections are an injection option for treatment for a number of conditions that are characterized by inflammation, which is the immune system's response to injury. The injections of cortisone are helpful for chronic pain that has been difficult to respond to with other treatments. \\
Cortisone shots are usually applied for a variety of conditions, including: \\
Inflammatory tendonitis, such as rotator cuff \\

 Question: Which of the following is NOT a use for cortis

    \textbf{Generation Speed: 856 tok/s}
    \end{generationbox}
}{}

\IfFileExists{completions/prompt8/vllm_ptp_350m_s16_20b_decay38146_pre_relaxed_1.tex}{
    \begin{generationbox}{PLM 350M Relaxed}
    \small
    isone injections are used for inflammation and other conditions related to the human body. Cortisone shots work by breaking the immune system to the area where your body is hurting. Cortisone shots can be used in surgery, chronic pain, and chronic cancer. Cortisone shots can be used to treat certain health and autoimmune problems, cancer, and arthritis. \\
How are the shots used? \\
The cortisone shot is applied in a vein with minimal to no discomfort and will be done by a doctor. It involves placing a steroid patch in the part of the body where a body part is involved. The shot is injected into a

    \textbf{Generation Speed: 2001 tok/s}
    \end{generationbox}
}{}

\IfFileExists{completions/prompt8/vllm_qwen3_4b_base.tex}{
    \begin{generationbox}{Qwen3 4B}
    \small
    isone or steroid injections are anti-inflammatory injections that can be used for a variety of conditions. Cortisone can also be taken internally in the form of oral medications or injections are usually placed directly at the site of pain or discomfort in order to reduce inflammation and improve mobility. Cortisone shots are often given to patients suffering from arthritis and other joint pain. \\
Steroid injections are used to treat pain or discomfort caused by arthritis in the hands, neck, knees, hips, back or shoulders, muscle pain and to accelerate the healing process of an injury, including tendonitis, bursitis and disc problems. They are also used for conditions

    \textbf{Generation Speed: 102 tok/s}
    \end{generationbox}
}{}

\IfFileExists{completions/prompt8/vllm_llama_3_2_1b.tex}{
    \begin{generationbox}{Llama 3.2 1B}
    \small
    isone is used when the patient has inflammation that is not responding to other medications. It is typically found in the upper leg or arm, or in the lower back with epidural injections. Usually, intravenous (IV) injections are given first. This method may provide an injection to the lowest place on the body where inflammation may have occurred. On average, steroid injections take about 10 to 15 minutes and may completely relieve the pain in the area. \\
This post is not intended for a diagnosis. Please use caution and seek professional medical help if the pain persists or is causing symptoms to affect your life. If you are considering steroid injections

    \textbf{Generation Speed: 316 tok/s}
    \end{generationbox}
}{}

\IfFileExists{completions/prompt8/vllm_smollm2_360m.tex}{
    \begin{generationbox}{SmolLM2 360M}
    \small
    isone is used painlessly to reduce swelling and inflammation, but only in a very low dose (whereas a cortisone shot is used in higher doses). You should see your doctor if you know you’re sensitive to cortisone (an allergy which causes an itchy or burning rash or fever), as cortisone shots have the potential to cause rebound conditions such as muscle swelling. However, doctors generally recommend avoiding cortisone shots for up to 5 years as repeated injections of cortisone in combination with regular injection of an anti-inflammatory medication can result in high blood pressure. 

What are the benefits of cort

    \textbf{Generation Speed: 589 tok/s}
    \end{generationbox}
}{}

\end{document}